\documentclass[11pt]{article}
\usepackage[final]{acl}
\usepackage{times}
\usepackage{latexsym}
\usepackage[T1]{fontenc}
\usepackage[utf8]{inputenc}
\usepackage{microtype}
\usepackage{inconsolata}
\usepackage{graphicx}
\usepackage{hyperref}     
\usepackage{url}            
\usepackage{booktabs}       
\usepackage{amsfonts}       
\usepackage{nicefrac}       
\usepackage{microtype}      
\usepackage{xcolor}         
\usepackage{multirow}
\usepackage{threeparttable}
\usepackage{adjustbox}
\usepackage{makecell}
\usepackage{graphicx}
\usepackage{multicol}
\usepackage{multirow}
\usepackage{subcaption}
\usepackage{float}
\usepackage[table]{xcolor}
\usepackage{amsmath,amssymb,amsthm,mathtools}
\definecolor{lightgrayrow}{gray}{0.93}

\definecolor{bestbg}{HTML}{D9EAD3}
\definecolor{secondbg}{HTML}{DDEBF7}

\newcommand{\best}[1]{\cellcolor{bestbg}\textbf{#1}}
\newcommand{\second}[1]{\cellcolor{secondbg}#1}
\newtheorem{theorem}{Theorem}

\newcommand{\Bset}{\mathcal{B}}

\newcommand{\Ldist}{L_{\mathrm{Bucket\text{-}SFT}}}
\newcommand{\Lhdpo}{L_{\mathrm{HDPO}}}

\newcommand{\sigmoid}{\sigma}
\newcommand{\pos}[1]{\left[#1\right]_+}
\usepackage[most]{tcolorbox}
\usepackage{tabularx}
\usepackage[most]{tcolorbox}
\usepackage{xcolor}
\usepackage{fancyvrb}

\usepackage[table]{xcolor}
\usepackage{booktabs}
\usepackage{multirow}
\usepackage{graphicx}
\usepackage{enumerate}
\usepackage{enumitem}
\usepackage{array}

\definecolor{headerrow}{HTML}{EAF3F8}     
\definecolor{subheaderrow}{HTML}{F5FAFD}  
\definecolor{stripegray}{HTML}{FAFAFA}    
\definecolor{bucketrow}{HTML}{FFF4D6}     
\definecolor{hdporow}{HTML}{EAF6EA}       
\definecolor{modelcell}{HTML}{F1F4F7}     
\definecolor{rulegray}{HTML}{BFC7CF}      
\definecolor{bucketrow}{gray}{0.92}

\arrayrulecolor{rulegray}
\renewcommand{\arraystretch}{1.12}
\definecolor{polyblue}{HTML}{1F4E79}
\definecolor{polylight}{HTML}{F3F8FC}
\definecolor{polyteal}{HTML}{0F766E}
\definecolor{polygray}{HTML}{4B5563}
\definecolor{pendinggray}{HTML}{777777}

\newcommand{\imp}[1]{{\textit{#1}}}

\newcolumntype{Y}{>{\raggedright\arraybackslash}X}
\newtcolorbox{takeawaybox}[1]{
  enhanced,
  breakable,
  colback=gray!6,
  colframe=gray!35,
  colbacktitle=gray!18,
  coltitle=black!75,
  boxrule=0.6pt,
  arc=2mm,
  left=1.2mm,
  right=1.2mm,
  top=0.8mm,
  bottom=0.8mm,
  fonttitle=\bfseries,
  title={Takeaway #1}
}

\title{\textit{PolyAlign}: Conditional Human-Distribution Alignment}

\author{
    L. D. M. S. Sai Teja\textsuperscript{1} \enspace 
    Ufaq Khan\textsuperscript{2} \enspace
    Sathira Silva\textsuperscript{2} \enspace \\
    \textbf{Xiao Wu}\textsuperscript{\textbf{2}} \enspace
    \textbf{Muhammad Haris Khan}\textsuperscript{2} \enspace \\
    \textsuperscript{1}NIT Silchar, India \enspace
    \textsuperscript{2}MBZUAI, Abu Dhabi, UAE \\
    \texttt{saitejalekkala05@gmail.com, \{ufaq.khan, muhammad.haris\}@mbzuai.ac.ae}
}

\begin{document}

\maketitle

\begin{abstract}
Standard supervised fine-tuning (SFT) and preference optimization often collapse diverse interactions into a single assistant style, suppressing natural variation across languages, tasks, and dialogue settings.
We study this problem as \emph{Conditional Human-Distribution Alignment}: models should match the human response distribution appropriate to the current interaction context, rather than a universal response style. We introduce \textbf{PolyAlign}, a distribution-aware alignment framework that organizes interaction data into bucket-specific human reference distributions defined by language, interaction track, response family, and length. \textbf{PolyAlign} combines Bucket-Aware SFT, which balances optimization across heterogeneous buckets, with Human-Distribution Preference Optimization (\textit{HDPO}), which regularizes preference learning using critic-estimated distance to bucket-specific human support. Across a bilingual evaluation suite covering English and Chinese single- and multi-turn settings, \textbf{PolyAlign} improves conditional naturalness and distributional faithfulness while preserving competitive task utility. The results\footnote{GitHub: \url{https://github.com/saitejalekkala33/PolyAlign.git}} suggest that post-training should move beyond global alignment objectives toward interaction-aware alignment with human response distributions.
\end{abstract}

\begin{figure}
    \centering
    \includegraphics[width=1.0\linewidth]{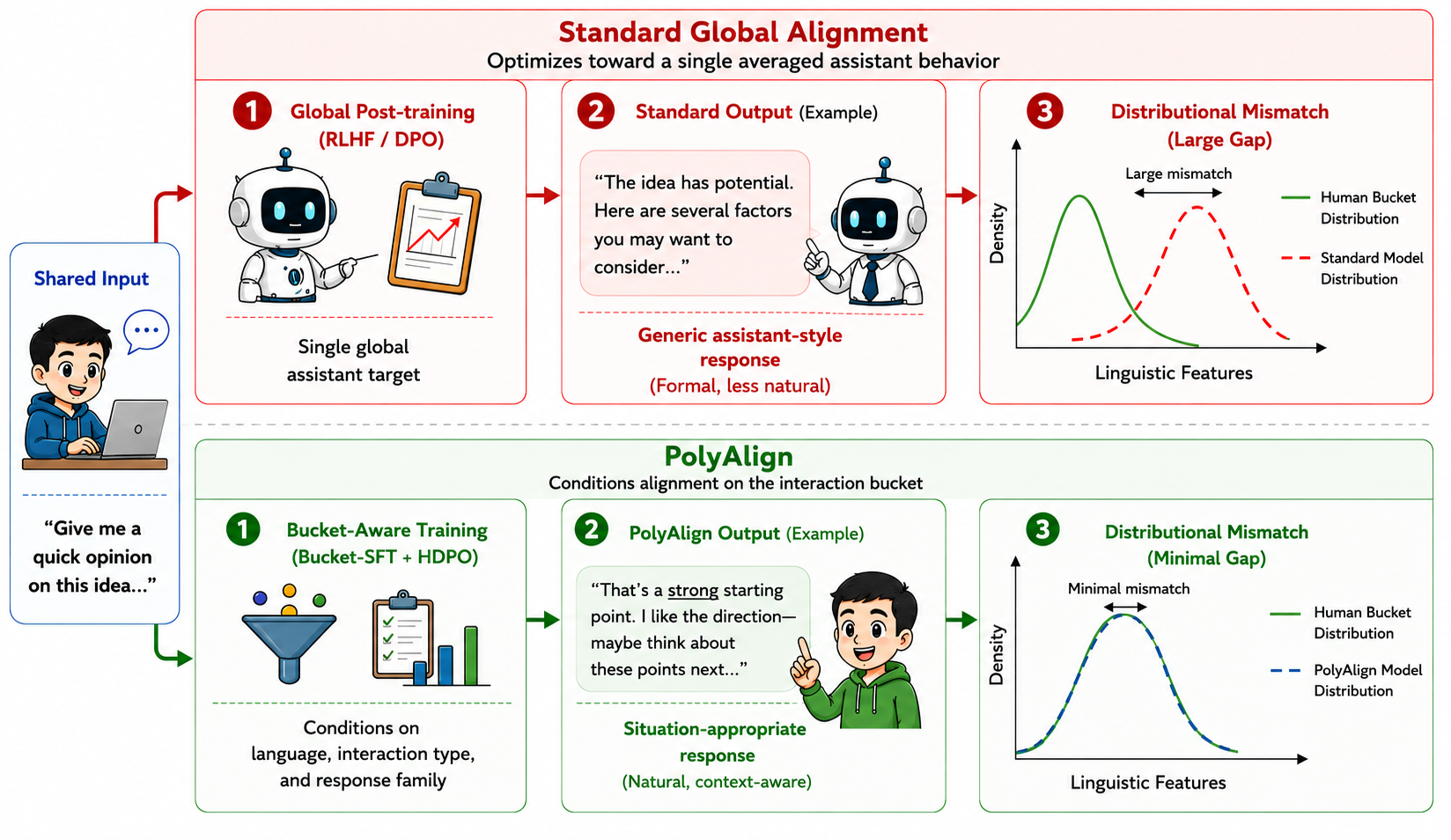}
    \caption{\textbf{PolyAlign vs. global alignment.} Unlike standard RLHF/DPO-style post-training, which can collapse diverse contexts into a generic assistant style, PolyAlign aligns responses to human distributions for more natural, situation-appropriate generation.}
    \label{fig:intro-fig}
\end{figure}

\section{Introduction}
\label{sec:introduction}

Large language models (LLMs) have become capable through large-scale pretraining, instruction tuning, and preference-based post-training. These methods have substantially improved helpfulness and instruction following, enabling modern assistant-style systems.
Scaling studies established that sufficiently large autoregressive models can perform a wide range of tasks from prompts alone \citep{brown2020language}, while instruction tuning showed that \textit{SFT} on diverse task instructions can substantially improve zero-shot generalization and interaction quality \citep{wei2021finetuned,sanh2021multitask,chung2024scaling,wang2023self}. Further progress came from alignment methods, based on human demonstrations, preference data, and reinforcement or preference optimization \citep{christiano2017deep,stiennon2020learning,ouyang2022training,bai2022constitutional,bai2022training,rafailov2023direct}. Standard post-training often pushes diverse interactions toward a single generic assistant style, rather than the kind of response humans would naturally give in each setting.

\noindent PolyAlign aligns model responses across multiple interaction settings by conditioning post-training on  human response distributions. We formulate alignment here as \emph{conditional human-distribution alignment}, shown in Fig~\ref{fig:intro-fig}. \textit{The \textbf{human distribution} is the set of responses that humans would naturally produce for a given interaction setting and how they vary in content, style, and form. This is represented using human responses and their linguistic-feature patterns within each bucket.} 

Preference optimization methods such as DPO, ORPO, SimPO, KTO, and RRHF demonstrate that post-training objectives can strongly reshape response behavior without requiring full RL pipelines \citep{rafailov2023direct,yuan2023rrhf,hong2024orpo,meng2024simpo,ethayarajh2024kto}.  Controllable generation has long emphasized that useful text systems should support movement across a family of target distributions rather than optimize a single generic objective \citep{keskar2019ctrl,dathathri2019plug,krause2021gedi,yang2021fudge,li2021prefix,lester2021power}. Distribution-matching views of post-training suggest that standard SFT can over-concentrate the distribution of learned generation, motivating more explicit objectives to target response distributions \citep{korbak2023pretraining,li2024entropic}. 

\paragraph{Contributions.}
\begin{enumerate}[leftmargin=3ex,noitemsep]
    \item We formulate \textbf{conditional human-distribution alignment}: post-training should match the human response distribution of each interaction context rather than enforce a single global assistant behavior. \imp{The goal is the right kind of answer for the right interaction.}
    \item We introduce \textbf{PolyAlign}: \textbf{Bucket-SFT} equalizes bucket mass to minimize macro bucket risk, while \textbf{HDPO} guides preference learning toward bucket-specific human support. \imp{PolyAlign makes both post-training interaction-conditioned.}
    \item We develop an evaluation protocol combining utility, bucket-level distributional metrics, the utility--naturalness frontier, human ratings, blinded bucket coherence, agreement, and LLM judges. \imp{It measures whether models remain useful while matching the appropriate human response distribution.}
\end{enumerate}

More broadly, our goal is to shift the alignment question from \textit{``how do we make a model better on average?''} $\rightarrow$ \textit{``how do we make a model produce the right kind of answer for the right kind of interaction?''} We view this as a natural next step for post-training: moving beyond generic alignment toward conditional, naturalistic, and interaction-aware alignment.
\section{Related Work}
\label{sec:related_work}
\noindent\textbf{Instruction tuning and preference-based alignment.}
LLM post-training has advanced from instruction tuning, which improved few-shot and zero-shot behavior \citep{brown2020language,wei2021finetuned,sanh2021multitask,chung2024scaling}, to synthetic and curated alignment corpora such as Self-Instruct, LIMA, and OpenAssistant \citep{wang2023self,zhou2023lima,kopf2023openassistant}. Human-feedback methods likewise evolved from RLHF pipelines \citep{christiano2017deep,stiennon2020learning,ouyang2022training,bai2022constitutional,bai2022training} to offline preference objectives including RRHF, DPO, ORPO, SimPO, and KTO \citep{yuan2023rrhf,rafailov2023direct,hong2024orpo,meng2024simpo,ethayarajh2024kto}. PolyAlign builds on these advances but shifts the target from a single global assistant behavior to interaction-specific human response distributions.         

\noindent\textbf{Structured alignment and controllable generation.}
Our work is related to approaches that treat alignment as structured rather than one-dimensional. Multi-attribute frameworks such as SteerLM and HelpSteer decompose helpfulness into multiple dimensions \citep{dong2023steerlm,wang2024helpsteer,wang2024helpsteer2}, while studies of diversified preferences show that feedback datasets can encode distinct alignment behaviors \citep{zeng2024diversified}. Controllable generation methods similarly steer models across behavior families rather than a single generic mode, as shown by CTRL, PPLM, GeDi, FUDGE, Prefix-Tuning, and Prompt Tuning \citep{keskar2019ctrl,dathathri2019plug,krause2021gedi,yang2021fudge,li2021prefix,lester2021power}. These works motivate our formulation, but instead of relying on manually specified attributes or globally aggregated preferences, we model human response distributions conditioned on language, interaction track, and response family.

\noindent\textbf{Distribution-aware post-training.} Our framework also builds on distributional views of post-training, where preference information is incorporated directly into language-model training \citep{korbak2023pretraining} and distribution-matching objectives are used to reduce the over-concentration often induced by standard SFT \citep{li2024entropic}. PolyAlign extends this perspective by targeting the appropriate human response distribution for each interaction bucket. Bucket-SFT performs bucket-aware supervised learning against human reference distributions, while HDPO adds bucket-aware weighting and a distribution-matching regularizer to offline preference optimization. We study this setting on compact open models such as Qwen2.5, Gemma~2, and Llama~3.2. The entire pipeline is given in Fig~\ref{fig:placeholder}.

\begin{figure*}[t]
    \centering
    \includegraphics[width=1.0\linewidth]{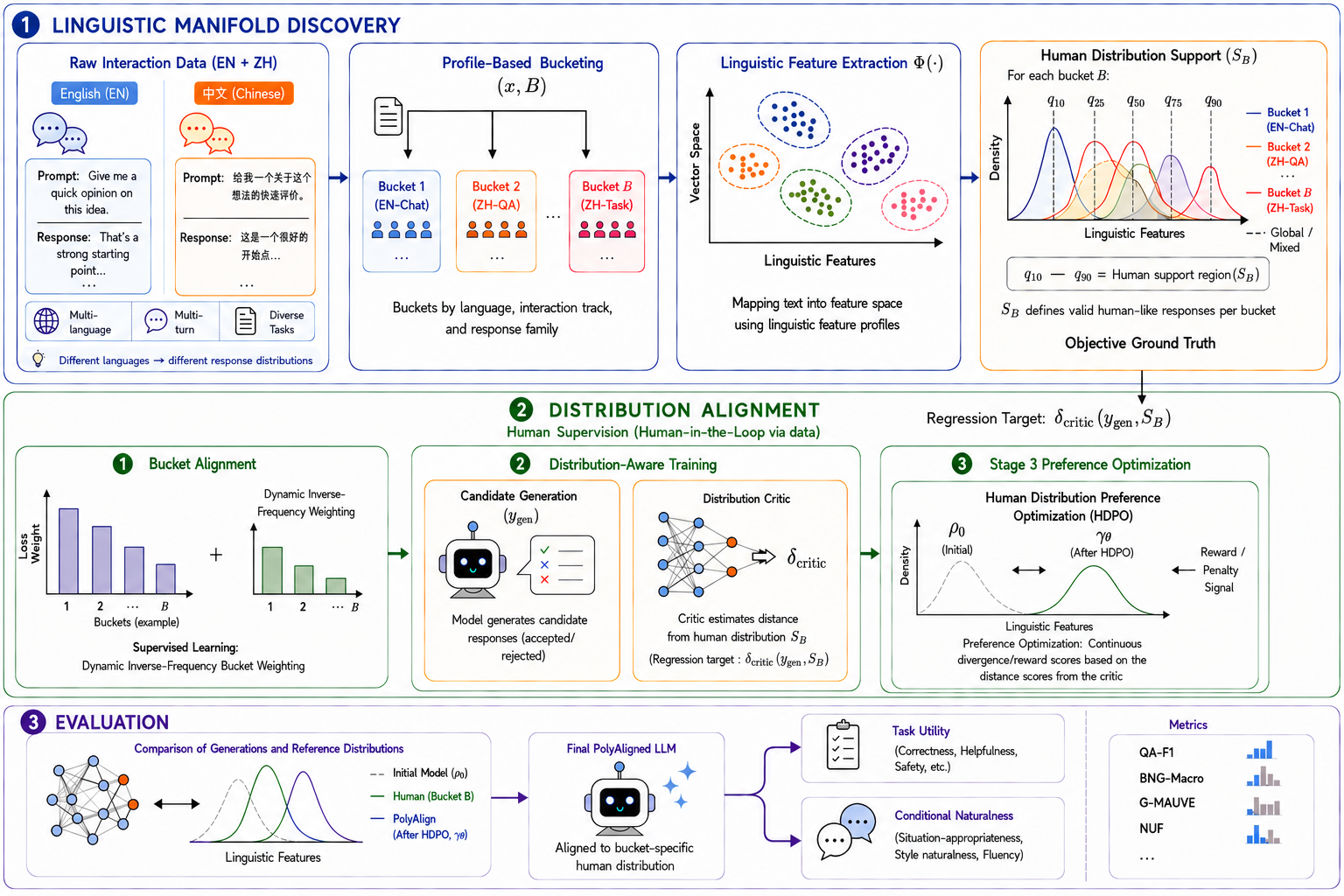}
    \caption{\textbf{PolyAlign pipeline.} PolyAlign organizes bilingual interaction data into bucket-specific human distributions, then aligns models through bucket-weighted SFT, critic-based distribution training, and HDPO. The final model is evaluated for task utility and conditional naturalness using QA-F1, BNG-Macro, G-MAUVE, and NUF.}
    \label{fig:placeholder}
\end{figure*}
\section{Methodology}
\label{sec:method}

\subsection{Bucketed Human Reference Distributions}
\label{sec:method_setup}
We align model outputs to the human distribution for each interaction setting rather than learning one global behavior. Buckets are defined a priori as $b=(\ell,t,f,r)$ by language, single- or multi-turn track, response family, and length. With a language-inclusive tokenizer, short, medium, long, and extra-long denote $\leq40$, $41$-$120$, $121$-$240$, and $>240$ tokens. Let $b_i\in\Bset$ be the bucket of instance $i$, $n_b$ its size, and $N:=\sum_{b\in\Bset}n_b$. We retain buckets with at least 20 human responses and estimate a reference distribution $\Lambda_b$ over features $z\in\mathbb{R}^d$. For each, we record size, dataset/style composition, feature-wise means, medians, standard deviations, and ${q_{10},q_{25},q_{50},q_{75},q_{90}}$. The default support is $[q_{10},q_{90}]$; $[q_{25},q_{75}]$ and stricter minimum support are sensitivity checks. We use $\Lambda_b$ to rebalance SFT and define a critic score, with lower values indicating closer human-distribution alignment. Table~\ref{tab:bucket-inventory} reports the summary of the bucket inventory for both En and Zh languages, along all the train, validation and test splits.



\subsection{Bucket-Aware Supervised Fine-Tuning}
\label{sec:dist_sft}

A central difficulty in this alignment is bucket imbalance: frequent interaction regimes dominate in standard SFT, even though when the objective is to balance alignment across buckets (Table~\ref{tab:dataset_inventory}).

Bucket-SFT addresses this by assigning equal total optimization mass to every bucket. Let $\ell_i(\theta)$ denote the per-example token-normalized loss under model parameters $\theta$. We assign the pre-serialization weight for each $i$ in bucket $b$:  $\widetilde{w}_b = \frac{N}{|\Bset|\,n_b}.$ The resulting Bucket-SFT objective is
\begin{equation}
\label{eq:dist_sft_objective}
\Ldist(\theta)
:=
\frac{\sum_{i=1}^{N} \widetilde{w}_{b_i}\,\ell_i(\theta)}
     {\sum_{i=1}^{N} \widetilde{w}_{b_i}}.
\end{equation}

\begin{theorem}[\textbf{\textit{Bucket-SFT Optimizes Exact Macro Bucket Risk}}]
\label{thm:dist_sft_macro}
For the weights in $\widetilde{w}_b$, the Bucket-SFT objective in Eq.~\eqref{eq:dist_sft_objective} satisfies
\begin{equation}
\label{eq:dist_sft_macro_risk}
\Ldist(\theta)
=
\frac{1}{|\Bset|}
\sum_{b \in \Bset}
\frac{1}{n_b}
\sum_{i:\, b_i=b} \ell_i(\theta).
\end{equation}
Equivalently, each bucket contributes exactly the same total optimization mass:
\begin{equation}
\label{eq:equal_bucket_mass}
\sum_{i:\, b_i=b} \widetilde{w}_{b_i}
=
\frac{N}{|\Bset|}
\qquad \text{for every } b \in \Bset.
\end{equation}
\end{theorem}

\begin{takeawaybox}{1}
Bucket-SFT does not merely reweight examples; it exactly converts supervised learning into macro bucket risk minimization, so every bucket receives equal optimization mass, Proof~\ref{app:proof_dist_sft_macro}.
\end{takeawaybox}

\subsection{Learning the Human Distribution}
\label{sec:human_distribution}

We next define the critic target used by HDPO, Appendix~\ref{app:critic-arc}. \textit{For bucket $b$, the critic $s_\phi(y,b)$ should assign low scores to responses whose feature vector stays within the bucket support region $\Lambda_b$, and higher scores to responses that deviate from it.} Let $z$ be the response feature vector, and let $J_b(z)$ be the set of feature indices shared by $z$, the bucket support, and the bucket feature statistics, with $J_b(z)\neq\varnothing$. For each $j\in J_b(z)$, let $[l_{bj},u_{bj}]$ be the support interval and define the normalization scale $s_{bj}:=\max\{u_{bj}-l_{bj},\,\mathrm{std}_{bj},\,\varepsilon\}, \qquad \varepsilon>0.$ The normalized bucket-support distance is then
\begin{equation}
\label{eq:bucket_support_distance}
D_b(z)
:=
\frac{1}{|J_b(z)|}
\sum_{j\in J_b(z)}
\frac{\pos{l_{bj}-z_j}+\pos{z_j-u_{bj}}}{s_{bj}}.
\end{equation}

\begin{theorem}[\textbf{\textit{Bucket-Support Distance Is a Continuous Relaxation of Human Membership}}]
\label{thm:support_distance}
For the distance in Eq.~\eqref{eq:bucket_support_distance}, the following hold:
\begin{align}
\label{eq:support_distance_properties}
&\textnormal{(i)} D_b(z)\ge 0, \nonumber\\
&\textnormal{(ii)} D_b(z)=0
\iff
z_j \in [l_{bj},u_{bj}]
\ \forall j\in J_b(z), \nonumber\\
&\textnormal{(iii)} D_b(z)\ \text{is continuous and piecewise linear in } z.
\end{align}
\end{theorem}

\begin{takeawaybox}{2}
Bucket-Support Distance $D_b(z)=0$ when all usable active features lie inside $\Lambda_b$, and increases as the response moves farther outside the bucket support, Proof~\ref{app:proof_support_distance}.
\end{takeawaybox}

\subsection{Human-Distribution Preference Optimization (HDPO)}
\label{sec:hdpo}

We now incorporate bucket-aware human-distribution information into preference optimization. Consider a training triple $(x,y^{+},y^{-})$, where $y^{+}$ is the chosen response, $y^{-}$ is the rejected response, and both belong to bucket $b$. HDPO augments a standard sigmoid-DPO objective with a critic-based regularizer that favors responses that are both preferred and closer to the corresponding human bucket support $\Lambda_{b}$.

\noindent Define the policy log-probability margin $\Delta_{\pi}$, and the reference-model margin $\Delta_{\mathrm{ref}}$. 
\begin{equation}
\label{eq:policy_margin}
\Delta_{\pi} := \log \pi_{\theta}(y^{+}\mid x) - \log \pi_{\theta}(y^{-}\mid x),
\end{equation}
\vspace{-10pt}
\begin{equation}
\label{eq:reference_margin}
\Delta_{\mathrm{ref}} := \log \pi_{\mathrm{ref}}(y^{+}\mid x) - \log \pi_{\mathrm{ref}}(y^{-}\mid x).
\end{equation}

\noindent The sigmoid-DPO loss is
\begin{equation}
\label{eq:dpo_loss}
L_{\mathrm{DPO}}(\theta)
=
-\log \sigmoid\!\Big(\beta(\Delta_{\pi}-\Delta_{\mathrm{ref}})\Big),
\end{equation}
where $\beta>0$ is the inverse-temperature parameter. To inject distributional information, we define $p_{\theta} := \sigmoid(\beta\Delta_{\pi}),$ and use the critic scores $s_{\phi}(y^{+},b)$ and $s_{\phi}(y^{-},b)$ to form the HDPO regularizer
\begin{equation}
\label{eq:hdpo_regularizer}
\small
R_{\mathrm{HDPO}}(\theta)
:=
p_{\theta}\, s_{\phi}(y^{+},b)
+
(1-p_{\theta})\, s_{\phi}(y^{-},b).
\end{equation}

\noindent This regularizer is small when the policy places probability mass on responses that the critic judges to be closer to the target human bucket support. The full HDPO objective then combines preference learning with critic-guided distribution matching:
\begin{equation}
\label{eq:hdpo_objective}
\scriptsize
\Lhdpo(\theta)
=
\frac{1}{K}\sum_{i=1}^{K}
w_i\Big(
L_{\mathrm{DPO}}^{(i)}(\theta)
+
\lambda_{\mathrm{hd}}\,R_{\mathrm{HDPO}}^{(i)}(\theta)
\Big),
\end{equation}
where $K$ is the number of preference pairs, $\lambda_{\mathrm{hd}}\ge 0$ controls the strength of the distributional regularizer, and $w_i>0$ denotes an optional bucket-aware training weight.

\begin{theorem}[\textbf{\textit{Distributional Alignment of the HDPO Regularizer with Sigmoid DPO}}]
\label{thm:hdpo_alignment}
Consider a single chosen/rejected pair $(y^{+},y^{-})$ in bucket $b$. We have $\Delta_{\pi}$ and $p_{\theta},$ and let $s_{\phi}(y^{+},b)$ and $s_{\phi}(y^{-},b)$ be fixed critic scores. For the regularizer in Eq.~\eqref{eq:hdpo_regularizer},
\begin{equation}
\label{eq:hdpo_reg_derivative}
\small
\frac{dR_{\mathrm{HDPO}}}{d\Delta_{\pi}} = \beta\,p_{\theta}(1-p_{\theta})
\Big(
s_{\phi}(y^{+},b)-s_{\phi}(y^{-},b)
\Big).
\end{equation}
Therefore, if $s_{\phi}(y^{+},b) < s_{\phi}(y^{-},b),$ then minimizing $R_{\mathrm{HDPO}}$ strictly favors increasing the chosen-vs-rejected margin $\Delta_{\pi}$. Furthermore, for the sigmoid-DPO loss in Eq.~\eqref{eq:dpo_loss},
\begin{equation}
\label{eq:dpo_derivative}
\frac{dL_{\mathrm{DPO}}}{d\Delta_{\pi}}
=
-\beta\Big(1-\sigmoid\big(\beta(\Delta_{\pi}-\Delta_{\mathrm{ref}})\big)\Big)
<0.
\end{equation}
Hence, whenever the critic judges the chosen response to be closer to the human bucket support than the rejected response, the sigmoid-DPO term and the HDPO regularizer are distributionally aligned: both push the policy toward larger $\Delta_{\pi}$.
\end{theorem}

\begin{takeawaybox}{3}
Under the sigmoid-DPO variant used in PolyAlign, HDPO is locally distributionally aligned with preference learning: when the chosen response is distributionally better, both objectives push the policy toward it, Proof~\ref{app:proof_hdpo_alignment}.
\end{takeawaybox}
\section{Evaluation}
\label{sec:evaluation}

We evaluate PolyAlign along two axes: (1) \textbf{Automatic Measures}, comprising task utility, conditional naturalness, and LLM-judge; and (2) \textbf{Human Evaluation}, with a blinded bucket-coherence study, human-judge scores, and their agreement.

\paragraph{Automatic Measures.}
We measure task utility using \textit{QA-F1}, which captures overlap with reference answers. Conditional naturalness is evaluated using the Bucket Normality Gap (\textit{BNG}), which averages normalized violations of feature support; \textit{G-MAUVE}, which compares generated and reference distributions within each conditioning regime \citep{pillutla2021mauve}; and the Naturalness--Utility Frontier (\textit{NUF}) hypervolume. We additionally report \textbf{\textit{Agg}}, the geometric mean of \textit{QA-F1}, $1/(1+\mathrm{BNG})$, Conditional MAUVE, and NUF. For LLM-as-a-judge evaluation, we use Qwen3-30B-A3B-Instruct-2507 and Llama-3.3-70B-Instruct as judges. Each response is rated from 1 to 5 on eight dimensions: task success, factual grounding, instruction following, reference alignment, conditional appropriateness, response shape and length, discourse naturalness, and safety. Scores are mapped to a 0-100 scale as $100(s-1)/4$. We report composites for \textit{Utility}, \textit{Conditional Naturalness}, and \textit{Distribution Faithfulness}. The full prompts and rubrics are in Appendix~\ref{app:llm-judge-prompt-rubric}.


\paragraph{Human Evaluation.}
We complement the automatic evaluation with three human-based analyses. We evaluate 100 outputs per method in each language, corresponding to 500 English and 500 Chinese outputs across the five methods. Each English output is independently rated by four annotators, while each Chinese output is rated by two annotators, on usefulness, answerability, and naturalness. Second, we conduct a blinded bucket-coherence study, we sampled 50 En and 50 Zh outputs from each of five and four situations, yielding 250 and 400 outputs respectively, and measure bucket fit, stylistic coherence, and bucket-identification accuracy. Finally, we assess the reliability of the collected ratings using ICC and Krippendorff's $\alpha$.

\section{Experiments and Results}


\paragraph{Benchmark.}
PolyAlign is studied in a bilingual setting, English and Chinese. Our goal is to align model outputs to the appropriate human-linguistic response distribution for the interaction setting at hand. We organize the corpus into five interaction situations: \textit{assistant-like}, \textit{longform-qa}, \textit{open-chat}, \textit{qa-search}, and \textit{task-dialogue}. 
All instances are canonicalized and deduplicated after normalization to prevent train-test leakage. More details in Appendix~\ref{app:data_details}.

\begin{table}[H]
\centering
\scriptsize
\resizebox{1.0\linewidth}{!}{%
\begin{tabular}{lllcllc}
\toprule
\textbf{Lang.} & \textbf{Dataset} & \textbf{Situation Category} & \textbf{Track} & \textbf{Split Counts} & \textbf{Count} & \textbf{Total} \\
\midrule
\multirow{8}{*}{en} & Dolly~\cite{conover2023free} & \textit{assistant-like} & \multirow{5}{*}{single} & 13,525 / 713 / 757 & 14,995 & \multirow{8}{*}{\rotatebox[origin=c]{90}{\shortstack{584,422 / 55,575 / \\53,855 = \textbf{693,852}}}} \\
 & ELI5~\cite{fan2019eli5} & \textit{longform-qa} & & 91,772 / 5,446 / 7,786 & 105,004 & \\
 & MS~MARCO~\cite{bajaj2016ms} & \textit{qa-search} & & 80,143 / 9,754 / 9,399 & 99,296 & \\
 & SQuAD~v2~\cite{rajpurkar2016squad} & \textit{qa-search} & & 117,444 / 12,832 / 11,870 & 142,146 & \\
 & Natural Questions~\cite{kwiatkowski2019natural} & \textit{qa-search} & & 90,133 / 5,081 / 5,017 & 100,231 & \\
\cmidrule(lr){2-6}
 & DailyDialog~\cite{li2017dailydialog} & \textit{open-chat} & \multirow{3}{*}{multi} & 37,377 / 3,774 / 3,681 & 44,832 & \\
 & CoQA~\cite{reddy2019coqa} & \textit{qa-search} & & 98,015 / 10,620 / 7,983 & 116,618 & \\
 & MultiWOZ~\cite{budzianowski2018multiwoz} & \textit{task-dialogue} & & 56,013 / 7,355 / 7,362 & 70,730 & \\
\midrule
\multirow{6}{*}{zh} & COIG-CQIA~\cite{bai2025coig} & \textit{assistant-like} & \multirow{5}{*}{single} & 9,536 / 509 / 579 & 10,624 & \multirow{6}{*}{\rotatebox[origin=c]{90}{\shortstack{86,336 / 10,887 / \\6,357 = \textbf{103,580}}}} \\
 & HC3-Chinese~\cite{guo2023close} & \textit{longform-qa} & & 19,990 / 1,152 / 1,058 & 22,200 & \\
 & CMRC2018~\cite{cui2019span} & \textit{qa-search} & & 10,142 / 3,219 / 1,002 & 14,363 & \\
 & DRCD~\cite{shao2018drcd} & \textit{qa-search} & & 26,936 / 3,524 / 3,493 & 33,953 & \\
 & DuReader~\cite{he2018dureader} & \textit{qa-search} & & 15,923 / 1,956 / -- & 17,879 & \\
\cmidrule(lr){2-6}
 & OASST2-zh~\cite{kopf2023openassistant} & \textit{open-chat} & multi & 3,809 / 527 / 225 & 4,561 & \\
\bottomrule
\end{tabular}
}
\caption{Bilingual PolyAlign corpus inventory with realized split counts shown as train / val / test.}
\label{tab:dataset_inventory}
\end{table}

\paragraph{Models.}
We conduct experiments on Qwen2.5-1.5B, Qwen2.5-3B~\citep{qwen2025qwen25technicalreport}, Gemma2-2B~\cite{team2024gemma}, Llama-3.2-3B~\citep{grattafiori2024llama}, to examine performance trends with increasing model scale, spanning approximately 1.5B to 3B parameters. 

\paragraph{Baselines.}
We perform the experiments by considering the baselines as Normal Inference (BaseLM), Chain-of-thoughts Prompting (CoT), Full Supervised Fine-Tuning of the models (Full-SFT), and Direct Preference Optimization (DPO). More details on compute are given in the Appendix~\ref{app:implementation-compute}.

\begin{table*}[t]
\centering
\scriptsize
\begin{tabular}{ll*{10}{c}}
\toprule
\multirow{2}{*}{\textbf{Model}} &
\multirow{2}{*}{\textbf{Method}} &
\multicolumn{5}{c}{\textbf{English -- en}} &
\multicolumn{5}{c}{\textbf{Chinese -- zh}} \\
\cmidrule(lr){3-7}
\cmidrule(lr){8-12}
& & QA-F1 ($\uparrow$) & BNG ($\downarrow$)
& G-MAUVE ($\uparrow$) & NUF ($\uparrow$) & Agg ($\uparrow$)
& QA-F1 ($\uparrow$) & BNG ($\downarrow$)
& G-MAUVE ($\uparrow$) & NUF ($\uparrow$) & Agg ($\uparrow$) \\
\midrule

& BaseLM
& 0.248 & 6.344 & 0.531 & 0.170 & 0.235
& 0.195 & 2.466 & 0.586 & 0.225 & 0.293 \\

& CoT
& 0.099 & 2.592 & 0.307 & 0.222 & 0.208
& 0.022 & 1.685 & 0.289 & 0.080 & 0.117 \\

& Full-SFT
& 0.317 & 5.012 & \best{0.947} & 0.547 & 0.371
& 0.358 & 0.917 & \second{0.918} & 0.458 & 0.529 \\

& DPO
& 0.355 & 1.259 & 0.847 & 0.580 & 0.527
& 0.238 & 16.056 & 0.511 & 0.163 & 0.184 \\
\cmidrule(lr){2-12}
& Bucket-SFT
& \second{0.444} & \best{0.427} & \second{0.939}
& \second{0.585} & \best{0.643}
& \best{0.479} & \best{0.346} & \best{0.968}
& \second{0.636} & \second{0.684} \\

\multirow{-6}{*}{\rotatebox[origin=c]{90}{\textbf{Qwen2.5-1.5B}}}
& HDPO
& \best{0.463} & \second{0.465} & 0.846
& \best{0.623} & \second{0.639}
& \second{0.450} & \second{0.375} & 0.851
& \best{0.832} & \best{0.694} \\

\midrule

& BaseLM
& 0.194 & 7708.6 & 0.784 & 0.270 & 0.048
& 0.108 & 13.152 & 0.364 & 0.059 & 0.113 \\

& CoT
& 0.057 & 4.876 & 0.135 & 0.066 & 0.096
& 0.010 & \second{1.923} & 0.321 & 0.025 & 0.072 \\

& Full-SFT
& 0.308 & 411.77 & \best{0.911} & \second{0.534} & 0.138
& \second{0.352} & \best{1.439} & \second{0.946}
& \second{0.418} & \best{0.489} \\

& DPO
& 0.310 & 28526.0 & \second{0.874} & 0.365 & 0.043
& 0.081 & 36.750 & 0.342 & 0.016 & 0.058 \\

\cmidrule(lr){2-12}
& Bucket-SFT
& \second{0.387} & \second{0.383} & 0.871 & 0.516
& \second{0.595}
& 0.324 & 39.365 & \best{0.967} & \best{0.443} & 0.242 \\

\multirow{-6}{*}{\rotatebox[origin=c]{90}{\textbf{Gemma-2-2B}}}
& HDPO
& \best{0.541} & \best{0.275} & 0.788
& \best{0.853} & \best{0.731}
& \best{0.555} & 23.494 & 0.852 & 0.188 & \second{0.245} \\

\midrule

& BaseLM
& 0.200 & 4.847 & 0.558 & 0.309 & 0.277
& 0.202 & 2.479 & 0.640 & 0.525 & 0.374 \\

& CoT
& 0.101 & 2.291 & 0.228 & 0.213 & 0.197
& 0.053 & \second{1.215} & 0.322 & 0.066 & 0.150 \\

& Full-SFT
& 0.398 & 649.49 & \best{0.918} & \second{0.644} & 0.137
& 0.447 & 2.932 & \second{0.949}
& \second{0.574} & \second{0.498} \\

& DPO
& 0.267 & 4.014 & 0.757 & 0.523 & 0.381
& 0.323 & 37.530 & 0.737 & 0.079 & 0.148 \\

\cmidrule(lr){2-12}
& Bucket-SFT
& \best{0.476} & \best{0.250} & \second{0.903}
& \best{0.652} & \best{0.688}
& \best{0.502} & \best{0.584} & \best{0.964}
& \best{0.656} & \best{0.669} \\

\multirow{-6}{*}{\rotatebox[origin=c]{90}{\textbf{Qwen2.5-3B}}}
& HDPO
& \second{0.418} & \second{0.490} & 0.779 & 0.474
& \second{0.567}
& \second{0.449} & 23.537 & 0.748 & 0.182 & 0.223 \\

\midrule

& BaseLM
& 0.227 & 9.773 & 0.808 & 0.352 & 0.278
& 0.297 & 18.23 & 0.725 & 0.342 & 0.249 \\

& CoT
& 0.100 & 2.430 & 0.315 & 0.186 & 0.203
& 0.011 & 4.431 & 0.255 & 0.016 & 0.055 \\

& Full-SFT
& 0.344 & 1.139 & \best{0.906} & \second{0.627} & 0.550
& 0.265 & 0.932 & \second{0.879} & 0.292 & 0.433 \\

& DPO
& 0.356 & 24.431 & 0.786 & \best{0.632} & 0.288
& 0.103 & 31.843 & 0.552 & 0.034 & 0.088 \\

\cmidrule(lr){2-12}
& Bucket-SFT
& \second{0.372} & \second{0.373} & \second{0.867}
& 0.439 & \second{0.566}
& \best{0.460} & \second{0.633} & \best{0.948}
& \second{0.613} & \second{0.636} \\

\multirow{-6}{*}{\rotatebox[origin=c]{90}{\textbf{Llama-3.2-3B}}}
& HDPO
& \best{0.405} & \best{0.360} & 0.750 & 0.600
& \best{0.604}
& \second{0.449} & \best{0.428} & 0.707
& \best{0.877} & \best{0.665} \\

\bottomrule
\end{tabular}%
\caption{Main bilingual results. Green and blue indicate the best and second-best results.
}
\label{tab:main_results_bilingual}
\end{table*}

\subsection{Does Conditional Alignment Improve the Utility-Naturalness Trade-off?}

\imp{A distribution-aware alignment should not improve naturalness by sacrificing the usefulness, high task utility alone is insufficient if all interaction settings fall into the same generic assistant style.} We interpret the results in Table~\ref{tab:main_results_bilingual} jointly. QA-F1 measures reference-based task utility, BNG measures deviation from the linguistic support of the appropriate human bucket, G-MAUVE measures distributional similarity, and NUF summarizes the utility-naturalness frontier. 
Across 8 model-language settings, a PolyAlign method obtains highest aggregate score 7 times, showing that conditional training improves utility-naturalness trade-off. CoT provides no consistent benefit, inference-time reasoning do not recover the appropriate human response distribution, while Full-SFT and DPO are less stable across buckets and few times produce large BNG values, indicating substantial out-of-support distributional drift. Because BNG is unbounded and variance-normalized, we interpret such values as failure cases and report them alongside G-MAUVE and NUF. These reveal language-specific differences: Full-SFT got the highest \textit{En} G-MAUVE, whereas Bucket-SFT in \textit{Zh}. Combined metrics show conditional post-training best balances task utility and interaction quality.

\subsection{Is Bucket Awareness Sufficient, and When Does HDPO Help?}
PolyAlign contains two conceptually different interventions. Bucket-SFT changes how supervised optimization mass is allocated, whereas HDPO uses critic-estimated distance to the human bucket support during preference optimization. We separate their contributions by comparing Bucket-SFT with Full-SFT and HDPO with standard DPO.
\textbf{Bucket-SFT provides most stable intervention.} Relative to Full-SFT Bucket-SFT, improves QA-F1 and the aggregate score in 7/8 model-language settings. This supports the macro-bucket objective: unlike standard SFT, which can be dominated by frequent interaction regimes, Bucket-SFT assigns equal optimization mass to each bucket and preserves conditional response behavior. It trails Full-SFT in \textit{En} but tops \textit{Zh} G-MAUVE across all models, showing bucket balancing thrives on highly heterogeneous data. \imp{Bucket-SFT is the most reliable PolyAlign component as it improves conditional alignment without a separate critic.}
\textbf{HDPO improves on DPO but is not uniformly preferable to Bucket-SFT and is critic-sensitive.}
Relative to DPO, HDPO achieves higher QA-F1, lower BNG, and a higher aggregate score in all 8 settings.
Compared with Bucket-SFT, HDPO improves the aggregate score and NUF in 5/8 settings, while Bucket-SFT retains higher G-MAUVE and HDPO shows BNG failures for some \textit{Zh}-checkpoints. HDPO offers higher upside but is less stable, depending on accurate bucket support and critic calibration, given the 95\%CI in Figure~\ref{fig:CI-plots}; \imp{a miscalibrated critic steers the policy toward a distorted proxy of the human distribution.}
We give failures in Table~\ref{tab:bng_diagnostics} and analyze them rather than selecting only in which HDPO wins. \textbf{Embedding-level view:} Figure~\ref{fig:embedding_umap_shifts} qualitatively shows that different training objectives shift the response embedding distributions and their centroids. Ablations details given in Appendix~\ref{app:ablation-sec}.

\begin{figure*}[t]
    \centering
    \begin{minipage}{0.49\linewidth}
        \centering
        \includegraphics[width=\linewidth]{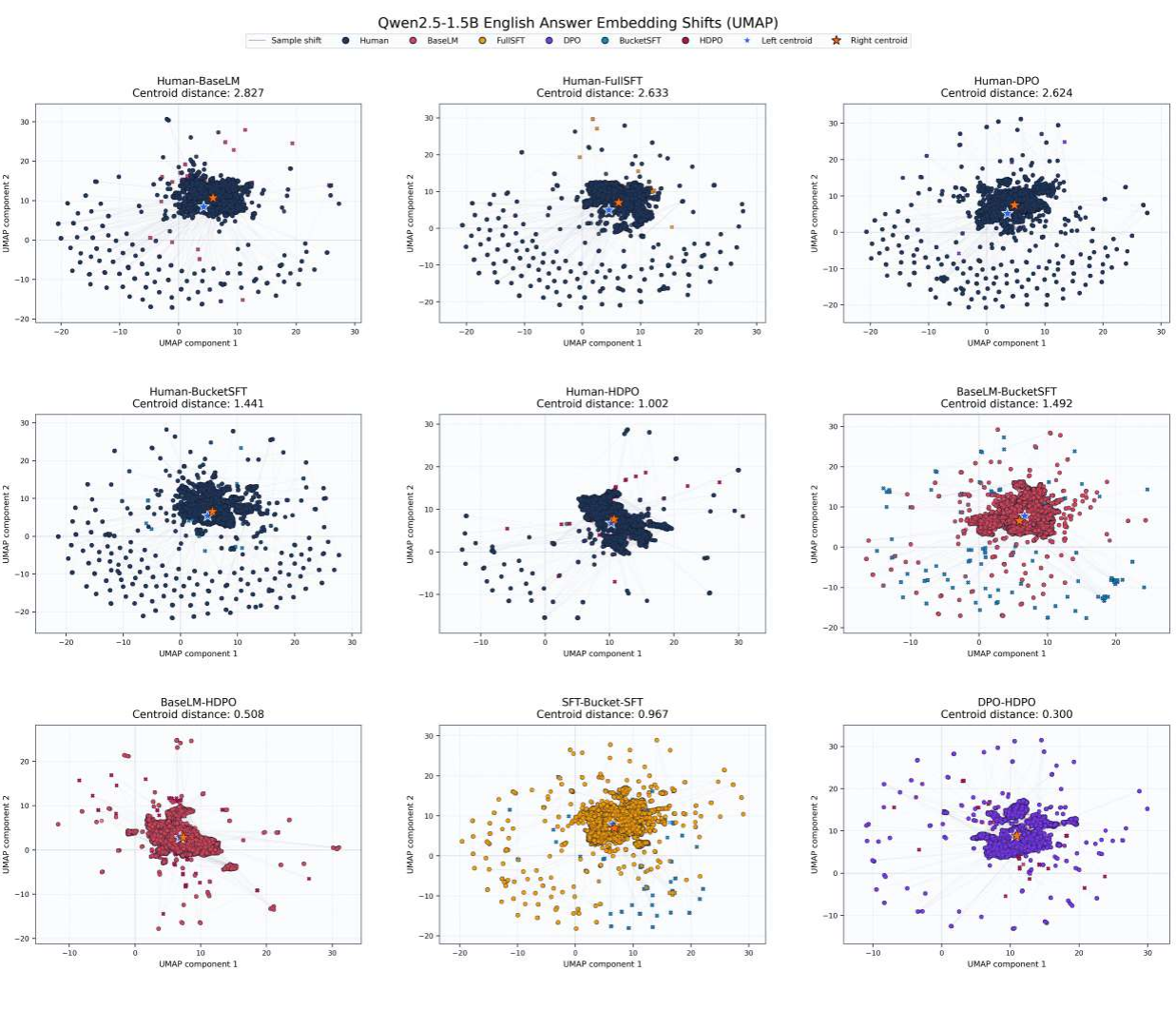}
        \vspace{0.5em}
        
        \small \textbf{(a)} English answer embedding shifts.
        \label{fig:umap-a}
    \end{minipage}
    \hfill
    \begin{minipage}{0.49\linewidth}
        \centering
        \includegraphics[width=\linewidth]{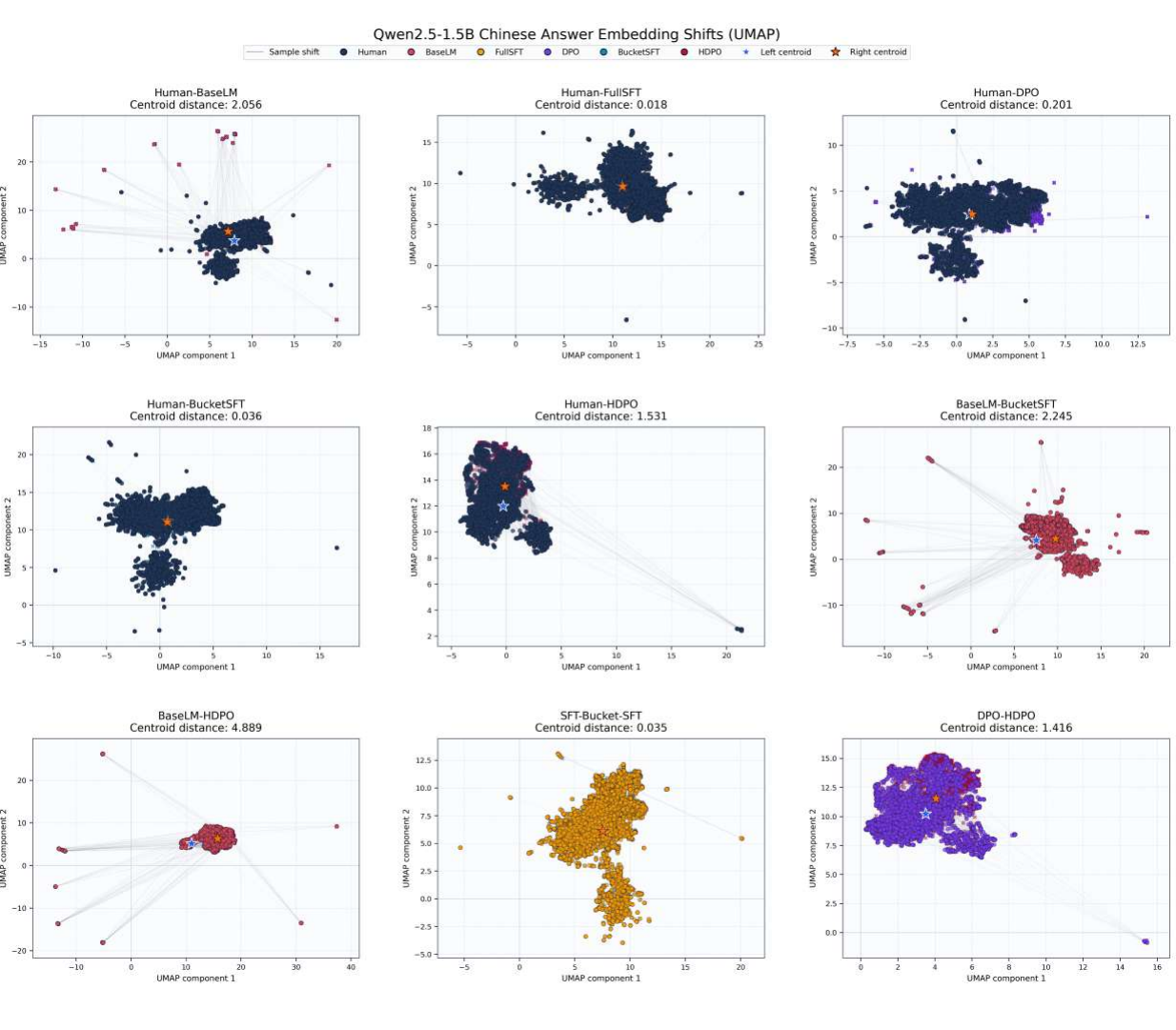}
        \vspace{0.5em}
        
        \small \textbf{(b)} Chinese answer embedding shifts.
        \label{fig:umap-b}
    \end{minipage}
    \caption{UMAP projections of Qwen2.5-1.5B answer embeddings across training methods. 
    }
    \label{fig:embedding_umap_shifts}
\end{figure*}

\subsection{Do Human Raters Recognize Conditional Alignment, and Are Their Judgements Reliable?}
\imp{Automatic metrics quantify conditional alignment, while human evaluation verifies whether these gains are perceptible in actual responses.} All methods are evaluated on the same 100 prompts per language, with 25 prompts from each model; \textit{En} outputs are assessed by 4 annotators and \textit{Zh} outputs by 2. Table~\ref{tab:combined-human-evaluation} shows a consistent ordering across usefulness, answerability, naturalness, bucket fit, style coherence, and blinded bucket identification. HDPO achieves the highest scores in both languages, although its relatively small improvement over Bucket-SFT indicates that most of the reliable gain comes from conditioning the supervised objective, with critic-guided preference optimization. Standard DPO performs below BaseLM across both languages, suggesting that unconstrained preference optimization can damage response quality and interaction-specific structure. From this we claim: \imp{preference optimization is most effective when constrained by the target interaction context,} since general response quality alone does not ensure conditional alignment. We have reported per-dataset results in Tables~\ref{tab:dataset_en}, \ref{tab:dataset_zh}.
The blinded coherence results further confirm that human raters can recognize the intended interaction regimes produced by conditional methods: Bucket-SFT and HDPO receive higher bucket-fit and style-coherence scores and are identified more accurately than generic baselines. Finally, the improvements from Full-SFT $\rightarrow$ Bucket-SFT $\rightarrow$ HDPO match the intended roles of PolyAlign's stages: Bucket-SFT first establishes differentiated response regimes, while HDPO further sharpens quality and consistency. \imp{Human evaluations are meaningful only if the reported method differences are not dominated by annotator variability.} Table~\ref{tab:human-agreement} reports (ICC(2,1)) and (ICC(2,$k$)) reliability. While \textit{En} (4 raters) yields stronger aggregate reliability than \textit{Zh} (2 raters), averaging improves stability across both. These support the reliability of the aggregated human scores while also cautioning against interpreting small differences between adjacent methods as exact individual-level preferences.

\begin{table}[h]
\centering
\resizebox{\columnwidth}{!}{%
\begin{tabular}{@{}llrrrrrrr@{}}
\toprule
& & \multicolumn{4}{c}{\textbf{Human Evaluation}} & \multicolumn{3}{c}{\textbf{Bucket Coherence}} \\
\cmidrule(lr){3-6} \cmidrule(l){7-9}
\textbf{Lang.} & \textbf{Method} & \textbf{Use} & \textbf{Ans} & \textbf{Nat} & \textbf{Avg.} & \textbf{Fit} & \textbf{Style} & \textbf{ID (\%)} \\
\midrule
\multirow{5}{*}{\rotatebox{90}{English}} 
& BaseLM     & 3.55 & 3.90 & 3.80 & 3.75 & 3.31 & 3.40 & 51.7 \\
& Full-SFT   & 4.40 & 4.60 & 4.50 & 4.50 & 3.76 & 3.84 & 64.2 \\
& DPO        & 3.20 & 3.60 & 3.10 & 3.30 & 3.14 & 3.22 & 47.5 \\
& Bucket-SFT & 4.60 & 4.80 & 4.75 & 4.72 & 4.29 & 4.34 & 78.3 \\
& HDPO       & \textbf{4.70} & \textbf{4.88} & \textbf{4.85} & \textbf{4.81} & \textbf{4.38} & \textbf{4.41} & \textbf{81.7} \\
\midrule
\multirow{5}{*}{\rotatebox{90}{Chinese}} 
& BaseLM     & 3.45 & 3.85 & 3.80 & 3.70 & 3.25 & 3.33 & 49.2 \\
& Full-SFT   & 4.35 & 4.58 & 4.48 & 4.47 & 3.70 & 3.78 & 61.7 \\
& DPO        & 3.05 & 3.55 & 2.95 & 3.18 & 3.02 & 3.13 & 44.2 \\
& Bucket-SFT & 4.60 & 4.80 & 4.80 & 4.73 & 4.24 & 4.28 & 76.7 \\
& HDPO       & \textbf{4.65} & \textbf{4.85} & \textbf{4.88} & \textbf{4.79} & \textbf{4.31} & \textbf{4.35} & \textbf{79.2} \\
\bottomrule
\end{tabular}%
}
\caption{Human evaluation ($n=100$ prompts per language; 4 raters for \textit{En}, 2 for \textit{Zh}) and blinded bucket-coherence results.}
\label{tab:combined-human-evaluation}
\end{table}


\begin{figure*}[t]
    \centering
    \begin{subfigure}[b]{0.48\textwidth}
        \centering
        \includegraphics[width=\textwidth]{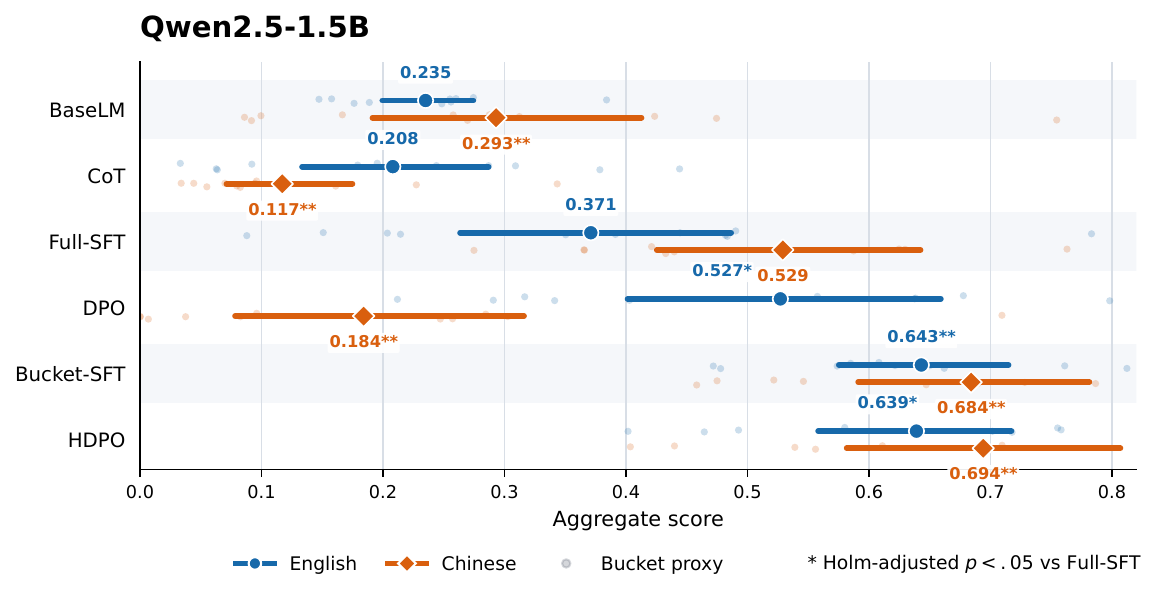}
        \caption{Bucket-SFT leads in \textit{En}; HDPO in \textit{Zh}.}
        \label{fig:sub1}
    \end{subfigure}
    \hfill
    \begin{subfigure}[b]{0.48\textwidth}
        \centering
        \includegraphics[width=\textwidth]{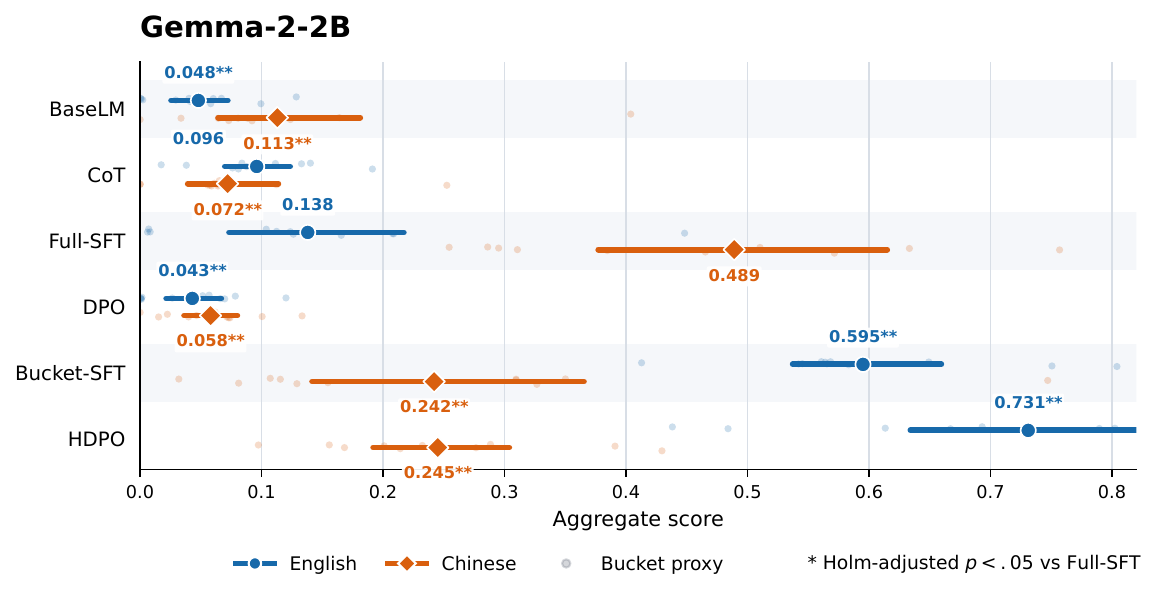}
        \caption{HDPO leads in \textit{En}; Full-SFT in \textit{Zh}.}
        \label{fig:sub2}
    \end{subfigure}
    
    \vspace{0.4cm}
    
    \begin{subfigure}[b]{0.48\textwidth}
        \centering
        \includegraphics[width=\textwidth]{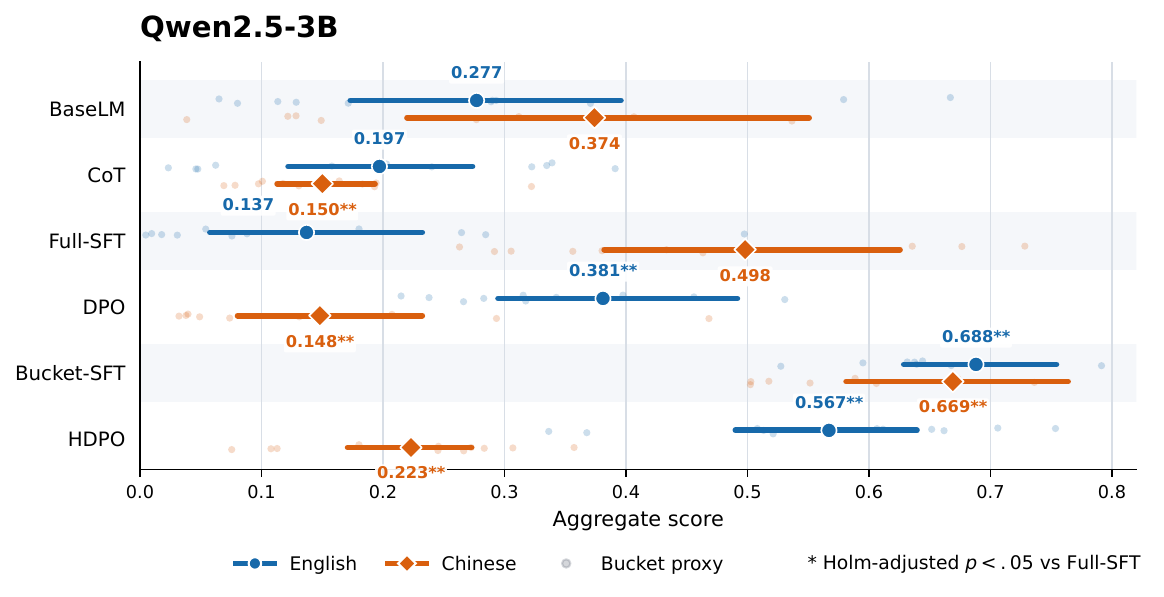}
        \caption{Bucket-SFT leads in both languages.}
        \label{fig:sub3}
    \end{subfigure}
    \hfill
    \begin{subfigure}[b]{0.48\textwidth}
        \centering
        \includegraphics[width=\textwidth]{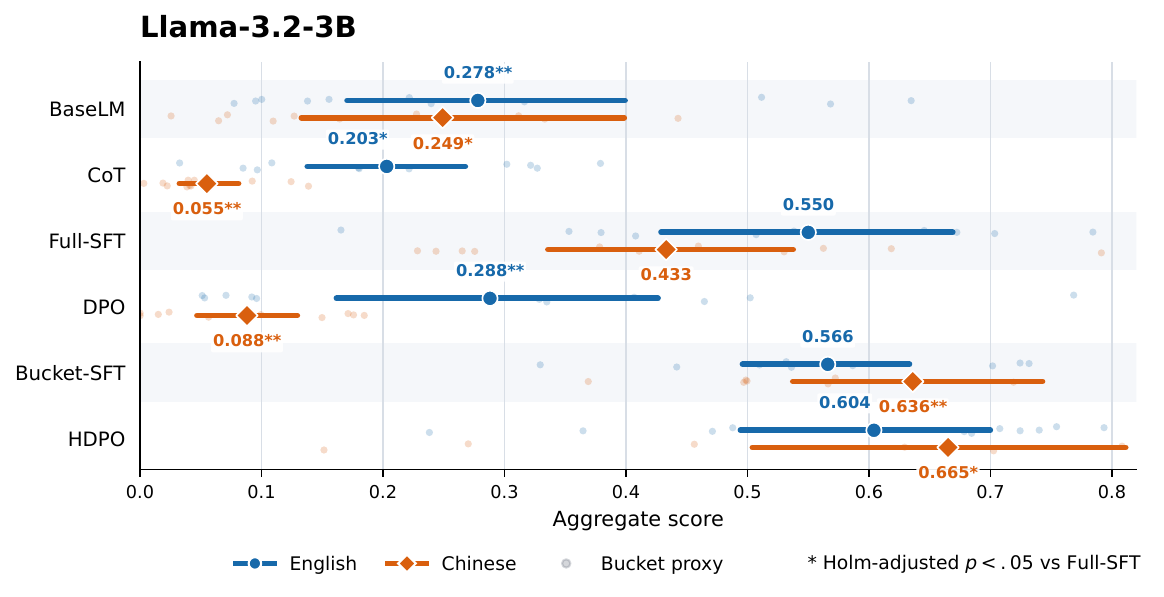}
        \caption{HDPO leads in both languages.}
        \label{fig:sub4}
    \end{subfigure}

    \caption{Bucket-level sensitivity of aggregate scores across four models, 95\% CI. 
    }
    \label{fig:CI-plots}
\end{figure*}

\subsection{Do LLM Judges Recover the Same Method Ordering as Humans?}
\imp{LLM-as-a-judge evaluation tests whether general-purpose judges recover the human-evaluation pattern.} Table~\ref{tab:llm-judge-results} shows that Full-SFT and Bucket-SFT are generally rated above BaseLM and DPO, with Bucket-SFT often matching or exceeding Full-SFT on conditional naturalness and distribution faithfulness. HDPO is more judge-sensitive: Llama rates it near the supervised methods, whereas Qwen assigns substantially lower utility and distribution-faithfulness scores, showing that judge choice materially affects conclusions about conditional alignment. The divergence of scores likely reflects a calibration mismatch rather than judge failure. \imp{General-purpose LLM judges may prefer responses that resemble a single reference or familiar assistant style, even when alternative outputs remain valid within the target human bucket.} This is especially relevant for HDPO, which preserves within-bucket variation. We therefore treat LLM-judge scores as complementary, not primary, evidence of conditional human-distribution alignment.

\subsection{$\downarrow$ AI-Detection mean $\uparrow$ Human-Likeness?}

We use three AI-text detectors as an exploratory analysis, Table~\ref{tab:ai_detection_joint_results}. Their scores vary substantially across several factors,
showing that detectability is not a stable measure of human-likeness. The joint score combines response quality with detector-based human-likeness so that low AI-detection rates are not rewarded when the underlying outputs are poor; under this measure, Bucket-SFT improves over Full-SFT in \textit{En} and \textit{Zh}, Figure~\ref{fig:quality-a}, while HDPO improves over DPO in both languages, Figure~\ref{fig:quality-b}. \imp{Low detectability is neither necessary nor sufficient for natural generation: perturbations can reduce detector confidence without improving quality, while appropriate responses may still be flagged as AI-generated.} AI detectors measure only surface patterns, not conditional naturalness.

\section{Conclusion}
We introduced PolyAlign for conditional human-distribution alignment, replacing a single global assistant target with interaction-specific response distributions. Across En and Zh, Bucket-SFT delivers the most consistent gains, while HDPO improves over standard DPO but remains sensitive to critic calibration and bucket quality. Human evaluations confirm that PolyAlign produces useful, and context-appropriate responses, while disagreement among LLM judges and AI-text detectors shows that conditional human-likeness cannot be reduced to generic preferences or surface detectability. Overall, interaction-aware post-training better preserves human response diversity without sacrificing utility.

\section*{Limitations}
This study is limited to English and Chinese, compact 1.5B--3B open models, and a partially crossed language-interaction corpus. Metadata-based buckets and hand-designed features may not capture all valid human variation, while HDPO remains sensitive to bucket quality and critic calibration. General-purpose LLM judges do not always reproduce the human ordering, and AI-text detectors are used only exploratorily because they measure surface detectability rather than usefulness or conditional naturalness. Future work should study larger multilingual and domain-diverse settings, richer learned bucket representations, improved critic calibration, broader human evaluation, and safety, cultural appropriateness, robustness, and fairness.

\section*{Potential Risks}
PolyAlign approximates human response distributions using existing datasets, metadata-defined buckets, and hand-designed linguistic features. These proxies may reproduce dataset biases, underrepresent valid linguistic or cultural variation, or be misinterpreted as universal norms of human behavior. Interaction-conditioned alignment could also be misused to generate persuasive or community-imitative responses; deployment should therefore retain independent safety safeguards and undergo language- and group-specific auditing. Our exploratory AI-text-detection analysis is not intended to facilitate detector evasion, and lower detectability should not be treated as a safety or quality objective. Critic training and multi-model evaluation add computational cost beyond standard~SFT.

\section*{Declaration of LLM usage}
During the preparation of this manuscript, ChatGPT was used for language polishing, grammar checking, and improving sentence clarity and readability, and also to review internally about the methodology, conduct analysis of the paper. All technical content, interpretations, figures, tables, and references were reviewed, verified, and finalized by the authors.

\bibliography{custom}
\clearpage
\appendix

\begin{center}
    {\Large \textbf{Appendix for PolyAlign}}
\end{center}

\section{Dataset Details and Normalization}
\label{app:data_details}

Table~\ref{tab:dataset_inventory} lists the datasets used in our bilingual English-Chinese setting and their corresponding interaction situations. The \textit{En} portion covers all five situations, while the \textit{Zh} portion covers four, with \textit{task-dialogue} currently instantiated only in \textit{En}. We therefore treat the corpus as a partially crossed language-situation design rather than a fully balanced factorial one. All datasets are normalized into a common schema containing the input, optional dialogue history, human target response, and structured metadata. We retain \textit{language}, \textit{track}, \textit{family}, \textit{style-bucket}, and \textit{length-bin} fields for every example. Fine-grained conditional regimes are defined from the tuple \textit{(language, track, family, length-bin)}. Split policies are dataset-specific. Official train/dev/test partitions are preserved when available, while train-only releases are partitioned deterministically. For multi-turn corpora, splitting is performed at the conversation or dialogue level rather than the turn level. After normalization, all examples are canonicalized and deduplicated using normalized input-response fields, with evaluation-safe precedence rules applied to prevent held-out examples from leaking into optimization data. Buckets with fewer than 20 held-out human examples are excluded from distributional evaluation to avoid unstable reference estimates.

\section{Implementation Details and Compute}
\label{app:implementation-compute}

\paragraph{Compute environment.}
All post-training and evaluation experiments were conducted on GPU nodes with eight AMD Instinct MI210 GPUs, each with 64~GiB of memory. We used PyTorch with ROCm support, Hugging Face Transformers, and LLaMA-Factory for full-parameter supervised fine-tuning and preference optimization. Unless otherwise specified, training used single-node distributed execution with one process per GPU.


\paragraph{Generation and automatic evaluation.}
For automatic evaluation, we generate responses from the final checkpoint of each method using the model's chat template. Unless otherwise stated, generation uses low-temperature decoding with temperature $0.1$, top-$p$ $0.9$, repetition penalty $1.1$, frequency penalty $0.1$, and presence penalty $0.0$. Generated responses are evaluated against the same held-out test instances and bucket metadata used for the corresponding language and model. Distributional metrics such as BNG and conditional MAUVE use the bucket-specific human reference statistics defined by the PolyAlign data construction.

\paragraph{LLM-as-a-judge evaluation.}
For rubric-based judge evaluation, we use Qwen3-30B-A3B-Instruct and Llama3.3-70B-Instruct as local vLLM-served judges. Since both judge models fit on a single MI300x 192GB GPU card, we use replica parallelism rather than tensor parallelism: each GPU hosts one independent judge replica, and evaluation examples are sharded across replicas. This setup is well suited to the judge workload, which consists of many independent scoring requests. Judge responses are parsed into the eight rubric dimensions described in Appendix~\ref{app:llm-judge-prompt-rubric} and aggregated into Overall, Utility, Conditional Naturalness, and Distribution Faithfulness scores. Each response is assigned integer scores from 1 to 5 across eight dimensions: task success, factual grounding, instruction following, reference alignment, conditional appropriateness, response shape and length, discourse naturalness, and safety, where 1 indicates severe failure, 3 indicates acceptable quality, and 5 indicates excellent quality. We map each score \(s\) to a 0--100 scale using \(100(s-1)/4\) and report weighted composite scores. The Overall score uses weights \((0.20, 0.15, 0.15, 0.10, 0.17, 0.08, 0.10, 0.05)\) across the eight dimensions; Utility emphasizes task correctness and grounding, Conditional Naturalness emphasizes bucket-appropriate style and discourse quality, and Distribution Faithfulness measures alignment with the target language, response family, style bucket, and length bin.

\paragraph{Human evaluation.} Human evaluation collects integer ratings from 1 to 5 for five criteria: task success, factual grounding, instruction following, reference alignment, and response quality. These criteria are aggregated into three metrics. Usefulness (\textsc{Use}) is computed as the mean of task success and factual grounding, Answerability (\textsc{Ans}) is computed as the mean of instruction following and reference alignment, and Naturalness (\textsc{Nat}) is represented by the response-quality score. All three metrics are reported on the original 1-5 scale, where 1 indicates severe failure, 3 indicates acceptable quality, and 5 indicates excellent quality.

\paragraph{Runtime reporting.}
All main experiments used a single node with 8 AMD Instinct MI210 GPUs (64~GiB each), PyTorch/ROCm, and LLaMA-Factory. LLM-as-a-Judge and ablation experiments used one MI300x 192GB GPU card and judges as Qwen3-30B-A3B-IT and Llama3.3-70B-IT were served with vLLM using one replica per GPU. Table~\ref{tab:compute_runtime_summary} reports wall-clock training times; BaseLM and CoT are omitted because they are inference-only. Times exclude queueing, checkpoint transfer, and evaluation post-processing.

\begin{table}[t]
\centering
\small
\setlength{\tabcolsep}{3pt}
\resizebox{\columnwidth}{!}{
\begin{tabular}{@{}c l rrrr@{}}
\toprule
\textbf{Model} & \textbf{Stage} & \multicolumn{2}{c}{\textbf{English}} & \multicolumn{2}{c}{\textbf{Chinese}} \\
\cmidrule(lr){3-4} \cmidrule(lr){5-6}
& & \textbf{Time} & \textbf{Steps} & \textbf{Time} & \textbf{Steps} \\
\midrule
\multirow{5}{*}{\rotatebox{90}{\textbf{Qwen2.5-1.5B}}}
 & Full-SFT    & 5h 35m  & 9132  & 55m    & 1349  \\
 & Bucket-SFT  & 6h 05m  & 9132  & 1h 05m & 1349  \\
 & DPO policy  & 14h 20m & 6685  & 2h 05m & 840   \\
 & HDPO critic & 8h 55m  & 68134 & 1h 30m & 10792 \\
 & HDPO policy & 15h 05m & 6685  & 2h 05m & 840   \\
\hline
\multirow{5}{*}{\rotatebox{90}{\textbf{Gemma-2-2B}}}
 & Full-SFT    & 7h 05m  & 9132  & 1h 20m & 1349  \\
 & Bucket-SFT  & 7h 45m  & 9132  & 1h 25m & 1349  \\
 & DPO policy  & 17h 50m & 6987  & 2h 50m & 915   \\
 & HDPO critic & 9h 10m  & 68134 & 1h 40m & 10792 \\
 & HDPO policy & 18h 54m & 6987  & 3h 02m & 915   \\
\hline
\multirow{5}{*}{\rotatebox{90}{\textbf{Qwen2.5-3B}}}
 & Full-SFT    & 8h 10m  & 9132  & 1h 30m & 1349  \\
 & Bucket-SFT  & 8h 30m  & 9132  & 1h 35m & 1349  \\
 & DPO policy  & 20h 05m & 6712  & 2h 40m & 745   \\
 & HDPO critic & 9h 05m  & 68134 & 1h 35m & 10792 \\
 & HDPO policy & 20h 42m & 6712  & 2h 43m & 745   \\
\hline
\multirow{5}{*}{\rotatebox{90}{\textbf{Llama-3.2-3B}}}
 & Full-SFT    & 8h 30m  & 9132  & 1h 35m & 1349  \\
 & Bucket-SFT  & 8h 45m  & 9132  & 1h 40m & 1349  \\
 & DPO policy  & 23h 11m & 7176  & 2h 55m & 767   \\
 & HDPO critic & 1h 35m  & 8517  & 1h 35m & 10792 \\
 & HDPO policy & 21h 36m & 7021  & 3h 03m & 767   \\
\bottomrule
\end{tabular}%
} 
\caption{Wall-clock training time for the full experimental grid. Times exclude queueing, checkpoint transfer, and evaluation post-processing.}
\label{tab:compute_runtime_summary}
\end{table}

\section{Further details on Evaluation metrics}
\label{sec:appendix-metrics}

Let $s \in \mathcal{S}$ denote an observed language-situation stratum, and let $b \in \mathcal{B}$ denote an eligible fine-grained bucket within that stratum. For each bucket $b$, let $H_b$ and $G_b$ denote the sets of human and model-generated responses. We report all distributional metrics as macro averages over eligible buckets.


\paragraph{Bucketed Naturalness Gap (BNG).} Let $\phi_k(y)$ denote the $k$-th linguistic feature extracted from response $y$, and let $\sigma_{b,k}$ be the empirical standard deviation of feature $k$ under the human bucket $H_b$. For a bucket $b$ with feature set $\mathcal{F}_b$, we define
\begin{equation}
\mathrm{BNG}(b)
=
\frac{1}{|\mathcal{F}_b|}
\sum_{k \in \mathcal{F}_b}
\frac{
W_1\!\left(\phi_k(G_b), \phi_k(H_b)\right)
}{
\max(\sigma_{b,k}, \varepsilon)
},
\end{equation}
where $W_1$ is the 1-Wasserstein distance and $\varepsilon > 0$. BNG measures how far the generated feature distribution drifts from the human feature distribution inside the correct bucket.

\paragraph{Conditional MAUVE.} To evaluate distributional similarity directly, we compute MAUVE between human and generated responses within each bucket: $\mathrm{C\text{-}MAUVE}$. We report global MAUVE computed over the pooled sets $H_b$ and $G_b$. 
\begin{equation}
    \mathrm{C\text{-}MAUVE}
=
\frac{1}{|\mathcal{B}|}
\sum_{b \in \mathcal{B}}
\mathrm{MAUVE}(H_b, G_b).
\end{equation}

\paragraph{Turn Dynamics Match (TDM).}
For multi-turn settings, we evaluate whether generated responses reproduce
human turn-level interaction dynamics. For each prompt--response pair $(x,y)$,
we measure the response-length ratio, lexical accommodation, and semantic
coupling:
\begin{equation}
\begin{aligned}
r(x,y) &= \frac{|y|}{\max(|x|,1)}, \\
a(x,y) &= \cos\!\bigl(\mathrm{tfidf}(x),\mathrm{tfidf}(y)\bigr), \\
c(x,y) &= \cos(z_x,z_y),
\end{aligned}
\label{eq:tdm-components}
\end{equation}
where $z_x$ and $z_y$ are latent representations obtained through an SVD
projection of TF-IDF features. For each multi-turn bucket $b$, we compute a variance-normalized Wasserstein
gap for each quantity:
\begin{equation}
\Delta_b^{u}
=
\frac{
W_1\!\left(u(G_b),u(H_b)\right)
}{
\max\!\left(\operatorname{Std}[u(H_b)],\varepsilon\right)
},
\qquad
u\in\{r,a,c\},
\label{eq:tdm-wasserstein}
\end{equation}
where $H_b$ and $G_b$ denote the human and generated responses in bucket $b$. We additionally compare latent prompt-response transition operators. Let
$Q_b$, $R_b^{H}$, and $R_b^{G}$ denote the centered latent prompt, human-response,
and generated-response matrices, respectively. The transition gap is
\begin{equation}
\Delta_b^{\mathrm{tr}}
=
\frac{
\left\|
\frac{Q_b^{\top}R_b^{G}}{n_b}
-
\frac{Q_b^{\top}R_b^{H}}{n_b}
\right\|_{F}
}{
\max\!\left(
\left\|
\frac{Q_b^{\top}R_b^{H}}{n_b}
\right\|_{F},
\varepsilon
\right)
}.
\label{eq:tdm-transition}
\end{equation}

Each gap is mapped to a bounded similarity using
$s(\Delta)=1/(1+\Delta)$. The final score is
\begin{footnotesize}
\begin{equation}
\mathrm{TDM}(b)
=
\frac{1}{4}
\left[
s(\Delta_b^{r})
+
s(\Delta_b^{a})
+
s(\Delta_b^{c})
+
s(\Delta_b^{\mathrm{tr}})
\right].
\label{eq:tdm}
\end{equation}
\end{footnotesize}

\paragraph{Naturalness--Utility Frontier (NUF).}
The preceding metrics capture complementary aspects of conditional alignment, whereas PolyAlign aims to jointly improve usefulness and conditional naturalness. We summarize this trade-off using the Naturalness-Utility Frontier (NUF). For each bucket $b$, utility and the bounded naturalness components are defined as
\begin{equation}
\begin{aligned}
u_b
&= \operatorname{clip}(\mathrm{F1}_b,0,1), \\
\widetilde{\mathrm{BNG}}_b
&= \frac{1}{1+\mathrm{BNG}(b)}, \\
\widetilde{\mathrm{C\text{-}MAUVE}}_b
&= \operatorname{clip}
   \bigl(\mathrm{C\text{-}MAUVE}_b,0,1\bigr), \\
\widetilde{\mathrm{TDM}}_b
&= \operatorname{clip}\bigl(\mathrm{TDM}_b,0,1\bigr),
\end{aligned}
\label{eq:nuf-components}
\end{equation}
where $\mathrm{F1}_b$ is the bucket-level QA-F1 score. We aggregate the naturalness components using a geometric mean:
\begin{equation}
n_b =
\left(
\widetilde{\mathrm{BNG}}_b
\widetilde{\mathrm{C\text{-}MAUVE}}_b
\right)^{1/2},
\label{eq:nuf-naturalness}
\end{equation}
and, for multi-turn buckets,
\begin{equation}
n_b^{\mathrm{mt}} =
\left(
\widetilde{\mathrm{BNG}}_b
\widetilde{\mathrm{C\text{-}MAUVE}}_b
\widetilde{\mathrm{TDM}}_b
\right)^{1/3}.
\label{eq:nuf-naturalness-mt}
\end{equation}
To unify both settings, let
\begin{equation}
n_b^{*} =
\begin{cases}
n_b^{\mathrm{mt}}, & \text{if $b$ is multi-turn},\\
n_b,               & \text{otherwise}.
\end{cases}
\label{eq:nuf-unified}
\end{equation}
Each bucket is then represented by $(u_b,n_b^{*})$ in
utility--naturalness space. Let $\mathcal{P}$ denote the set of non-dominated
bucket points. We summarize the resulting Pareto frontier by its hypervolume
relative to $(0,0)$:
\begin{equation}
\mathrm{NUF\text{-}HV}
=
\operatorname{Hypervolume}
\left(
\mathcal{P},
(0,0)
\right).
\label{eq:nuf-hv}
\end{equation}
\section{Critic Architecture and Preference-Pair Construction}
\label{app:critic-arc}
The critic encodes responses with a Transformer using a maximum length of 512 tokens and attention-masked mean pooling, then concatenates the response vector with a learned 64-dimensional bucket embedding to enable bucket-dependent support judgments. Scores are produced by a nonlinear head, $\mathrm{Linear}(d+64,256)$--GELU--Dropout$(0.1)$--$\mathrm{Linear}(256,256)$--GELU--Dropout$(0.1)$--$\mathrm{Linear}(256,1)$, trained with equally weighted regression and pairwise-ranking losses using a margin of 0.1. Preference pairs use the normalized human response as chosen and the aligned model prediction from the same source index as rejected, ensuring prompt-matched comparisons. Empty, identical, missing, source-mismatched, and out-of-range pairs are removed, while source and bucket metadata are retained for stratified auditing.
\section{HDPO Failure Modes}
\textbf{When critic guidance can fail:}
The alignment result in Theorem~3 is conditional: it applies when the critic correctly orders the chosen and rejected responses. Three departures are important in practice. First, a miscalibrated critic may predict distances that are systematically compressed, shifted, or incorrectly ranked. Second, small buckets yield uncertain feature quantiles and variances, so the support region may reflect sampling noise rather than a stable human distribution. Third, proximity to feature support is not identical to task utility: a response can match the expected length and style while omitting information needed to answer the request. In that case the DPO and distributional terms encode conflicting preferences. 
\textbf{Diagnosing extreme BNG values:}
BNG is an unbounded, variance-normalized distance, so a feature whose human within-bucket variance is near zero can contribute a very large value after a
modest absolute shift. The BNG diagnostic Table~\ref{tab:bng_diagnostics} shows that the resulting failures are not all of one kind. For Chinese Gemma-2-2B Bucket-SFT, removing the worst bucket reduces macro BNG from 39.365 to 0.468; the maximum comes from a 37-example multi-turn dialogue bucket and is a concentrated, sparse-support failure. In contrast, Chinese Gemma-2-2B HDPO remains at 23.473 after removing its maximum bucket (23.494 before removal), and Qwen2.5-3B HDPO remains at 23.500 (23.537 before removal). These two HDPO failures are broad across buckets and cannot be explained by a single outlier. English Gemma-2-2B DPO combines both phenomena: its maximum bucket reaches 280630.663, while macro BNG remains 3316.354 after that bucket is removed. Extreme BNG should be read as a diagnostic flag rather than a linear measure of practical severity. We report median, maximum, support size, and macro-without-maximum alongside the macro mean. 
\section{Ablation Studies}
\label{app:ablation-sec}
The ablation results, Table~\ref{tab:ablations}, show that Bucket-SFT + DPO achieves a higher aggregate score in seven of eight model-language settings, with especially large gains in BNG and overall robustness. Full-SFT + HDPO is competitive only for Gemma-2-2B in Chinese and occasionally improves G-MAUVE or NUF. Overall, bucket-aware training provides more consistent performance across models and languages.

\begin{table*}
\centering
\scriptsize
\setlength{\tabcolsep}{3.5pt}
\begin{tabular}{@{}llrrrr@{}}
\toprule
\textbf{Lang.} & \textbf{Agreement}
& \textbf{Use} & \textbf{Ans} & \textbf{Nat} & \textbf{Avg.} \\
\midrule
\multirow{3}{*}{\rotatebox{90}{\makecell{English\\($k=4$)}}}
& ICC(2,1)                & 0.74 & 0.80 & 0.77 & 0.83 \\
& ICC(2,4)                & 0.92 & 0.94 & 0.93 & 0.95 \\
& Krippendorff's $\alpha$ & 0.73 & 0.79 & 0.76 & 0.83 \\
\midrule
\multirow{3}{*}{\rotatebox{90}{\makecell{Chinese\\($k=2$)}}}
& ICC(2,1)                & 0.67 & 0.75 & 0.71 & 0.79 \\
& ICC(2,2)                & 0.80 & 0.86 & 0.83 & 0.88 \\
& Krippendorff's $\alpha$ & 0.65 & 0.73 & 0.69 & 0.76 \\
\bottomrule
\end{tabular}
\caption{
Inter-rater agreement for the human evaluation.}
\label{tab:human-agreement}
\end{table*}
\begin{table*}[t]
\centering
\scriptsize
\setlength{\tabcolsep}{3.5pt}
\renewcommand{\arraystretch}{1.02}

\resizebox{\textwidth}{!}{%
\begin{tabular}{@{}lrrrrcllc@{}}
\toprule
\textbf{Bucket $(t\!-\!f\!-\!r)$}
& \textbf{Train}
& \textbf{Val}.
& \textbf{Test}
& \textbf{Total}
& \textbf{Length $q_{10}/q_{50}/q_{90}$}
& \textbf{Datasets}
& \textbf{Style metadata}
& \textbf{Eligible} \\
\midrule

\multicolumn{9}{@{}l}{\textbf{English}} \\
\addlinespace[1pt]

multi-dialogue-long
& 14 & 2 & 2 & 18 & 126/142/172
& \textit{daily-dialog}
& open-chat
& No \\

multi-dialogue-medium
& 1,410 & 137 & 167 & 1,714 & 42/47/70
& \textit{daily-dialog}, \textit{multi-woz}
& open-chat, task-dialogue
& Yes \\

multi-dialogue-short
& 91,963 & 10,990 & 10,874 & 113,827 & 6/15/28
& \textit{multi-woz}, \textit{daily-dialog}
& task-dialogue, open-chat
& Yes \\

multi-dialogue-xlong
& 3 & 0 & 0 & 3 & 267/267/289
& \textit{daily-dialog}
& open-chat
& No \\

multi-qa-long
& 1 & 0 & 0 & 1 & 142/142/142
& \textit{coqa}
& qa-search
& No \\

multi-qa-medium
& 26 & 4 & 0 & 30 & 43/54/85
& \textit{coqa}
& qa-search
& No \\

multi-qa-short
& 97,987 & 10,616 & 7,983 & 116,586 & 1/2/6
& \textit{coqa}
& qa-search
& Yes \\

multi-qa-xlong
& 1 & 0 & 0 & 1 & 424/424/424
& \textit{coqa}
& qa-search
& No \\

single-assistant-long
& 1,407 & 92 & 79 & 1,578 & 127/159/214
& \textit{dolly}
& assistant-like
& Yes \\

single-assistant-medium
& 4,429 & 226 & 246 & 4,901 & 45/69/106
& \textit{dolly}
& assistant-like
& Yes \\

single-assistant-short
& 7,043 & 361 & 390 & 7,794 & 3/17/34
& \textit{dolly}
& assistant-like
& Yes \\

single-assistant-xlong
& 646 & 34 & 42 & 722 & 257/340/703
& \textit{dolly}
& assistant-like
& Yes \\

single-qa-long
& 57,554 & 3,279 & 3,314 & 64,147 & 127/158/214
& \textit{natural-questions}, \textit{eli5-category}, \textit{ms-marco}
& qa-search, longform-qa
& Yes \\

single-qa-medium
& 98,576 & 6,090 & 6,191 & 110,857 & 47/76/110
& \textit{natural-questions}, \textit{eli5-category},
  \textit{ms-marco}, \textit{squad-v2}
& qa-search, longform-qa
& Yes \\

single-qa-short
& 207,022 & 22,555 & 21,257 & 250,834 & 0/3/28
& \textit{squad-v2}, \textit{ms-marco},
  \textit{eli5-category}, \textit{natural-questions}
& qa-search, longform-qa
& Yes \\

single-qa-xlong
& 16,340 & 1,189 & 935 & 18,464 & 250/310/549
& \textit{eli5-category}, \textit{natural-questions}
& longform-qa, qa-search
& Yes \\

\addlinespace[2pt]
\midrule
\multicolumn{9}{@{}l}{\textbf{Chinese}} \\
\addlinespace[1pt]

multi-dialogue-long
& 535 & 85 & 37 & 657 & 130/171/228
& \textit{oasst2-zh}
& open-chat
& Yes \\

multi-dialogue-medium
& 785 & 83 & 51 & 919 & 46/71/110
& \textit{oasst2-zh}
& open-chat
& Yes \\

multi-dialogue-short
& 1,184 & 161 & 66 & 1,411 & 3/16/34
& \textit{oasst2-zh}
& open-chat
& Yes \\

multi-dialogue-xlong
& 1,305 & 198 & 71 & 1,574 & 279/422/743
& \textit{oasst2-zh}
& open-chat
& Yes \\

single-assistant-long
& 387 & 20 & 17 & 424 & 128/147/227
& \textit{coig-cqia}
& assistant-like
& No \\

single-assistant-medium
& 855 & 41 & 48 & 944 & 49/73/109
& \textit{coig-cqia}
& assistant-like
& Yes \\

single-assistant-short
& 384 & 16 & 58 & 458 & 9/11/36
& \textit{coig-cqia}
& assistant-like
& Yes \\

single-assistant-xlong
& 7,910 & 432 & 456 & 8,798 & 339/642/2292
& \textit{coig-cqia}
& assistant-like
& Yes \\

single-qa-long
& 8,990 & 528 & 520 & 10,038 & 133/187/206
& \textit{hc3-chinese}, \textit{dureader}, \textit{cmrc2018}
& longform-qa, qa-search
& Yes \\

single-qa-medium
& 7,510 & 607 & 499 & 8,616 & 47/74/110
& \textit{hc3-chinese}, \textit{cmrc2018},
  \textit{dureader}, \textit{drcd}
& longform-qa, qa-search
& Yes \\

single-qa-short
& 54,355 & 8,606 & 4,405 & 67,366 & 2/4/16
& \textit{drcd}, \textit{dureader},
  \textit{cmrc2018}, \textit{hc3-chinese}
& qa-search, longform-qa
& Yes \\

single-qa-xlong
& 2,136 & 110 & 129 & 2,375 & 252/339/622
& \textit{hc3-chinese}, \textit{dureader}, \textit{cmrc2018}
& longform-qa, qa-search
& Yes \\

\bottomrule
\end{tabular}%
}

\caption{
Bucket inventory for English and Chinese splits. 
}
\label{tab:bucket-inventory}
\end{table*}

\definecolor{rowbg}{HTML}{DDE2EC}

\begin{table*}[h]
\centering
\scriptsize
\setlength{\tabcolsep}{1pt} 
\renewcommand{\arraystretch}{1.10}
\begin{tabular}{llcccccccccccccccc}
\toprule
\multirow{3}{*}{\textbf{Model}} & \multirow{3}{*}{\textbf{Method}} & \multicolumn{8}{c}{\textbf{English}} & \multicolumn{8}{c}{\textbf{Chinese}} \\
\cmidrule(lr){3-10} \cmidrule(lr){11-18}
& & \multicolumn{4}{c}{\textbf{Qwen3-30B}} & \multicolumn{4}{c}{\textbf{Llama-3.3-70B}} & \multicolumn{4}{c}{\textbf{Qwen3-30B}} & \multicolumn{4}{c}{\textbf{Llama-3.3-70B}} \\
\cmidrule(lr){3-6} \cmidrule(lr){7-10} \cmidrule(lr){11-14} \cmidrule(lr){15-18}
& & \makecell{Utility} & \makecell{Cond.} & \makecell{D-Faith.} & \makecell{Overall} & \makecell{Utility} & \makecell{Cond.} & \makecell{D-Faith.} & \makecell{Overall} & \makecell{Utility} & \makecell{Cond.} & \makecell{D-Faith.} & \makecell{Overall} & \makecell{Utility} & \makecell{Cond.} & \makecell{D-Faith.} & \makecell{Overall} \\
\midrule
& BaseLM & 68.4 & 77.8 & 73.0 & 73.1 & 66.6 & 74.1 & 70.8 & 70.8 & 63.8 & 69.4 & 65.7 & 66.5 & 63.0 & 67.3 & 64.9 & 65.5 \\
& Full-SFT & 70.1 & \underline{81.2} & \underline{76.4} & \underline{75.4} & \underline{67.1} & 75.4 & 72.0 & \underline{71.5} & \underline{81.7} & \underline{87.2} & \underline{84.1} & \underline{84.4} & 80.1 & \underline{84.9} & \underline{82.8} & \underline{82.8} \\
& DPO & 69.9 & 80.8 & 75.8 & 75.0 & 66.0 & 74.4 & 71.0 & 70.5 & 62.5 & 67.0 & 63.6 & 64.8 & 60.7 & 64.5 & 62.3 & 62.8 \\
\rowcolor{rowbg} & Bucket-SFT & \cellcolor{bestbg}\textbf{73.1} & \cellcolor{bestbg}\textbf{82.7} & \cellcolor{bestbg}\textbf{78.5} & \cellcolor{bestbg}\textbf{78.0} & \cellcolor{bestbg}\textbf{69.0} & \underline{76.5} & \underline{73.6} & \cellcolor{bestbg}\textbf{73.0} & \cellcolor{bestbg}\textbf{82.4} & \cellcolor{bestbg}\textbf{87.8} & \cellcolor{bestbg}\textbf{84.6} & \cellcolor{bestbg}\textbf{85.1} & \cellcolor{bestbg}\textbf{80.4} & 85.0 & 83.0 & 83.0 \\
\rowcolor{rowbg} \multirow{-5}{*}{\rotatebox[origin=c]{90}{\textbf{Qwen2.5-1.5B}}} & HDPO & 46.8 & 55.6 & 49.1 & 51.2 & 66.8 & \cellcolor{bestbg}\textbf{77.2} & \cellcolor{bestbg}\textbf{74.2} & 70.5 & 56.9 & 65.5 & 58.8 & 59.8 & \underline{79.6} & \cellcolor{bestbg}\textbf{86.2} & \cellcolor{bestbg}\textbf{83.4} & \cellcolor{bestbg}\textbf{83.3} \\
\midrule
& BaseLM & 57.0 & 72.4 & 65.6 & 64.2 & 54.0 & 65.1 & 60.5 & 60.0 & 51.6 & 58.0 & 53.1 & 54.8 & 51.2 & 56.0 & 52.7 & 54.0 \\
& Full-SFT & 66.9 & 77.9 & 72.8 & 72.2 & 64.0 & \underline{72.2} & \underline{68.7} & \underline{68.4} & \cellcolor{bestbg}\textbf{75.8} & \cellcolor{bestbg}\textbf{82.8} & \cellcolor{bestbg}\textbf{78.8} & \cellcolor{bestbg}\textbf{79.1} & \cellcolor{bestbg}\textbf{74.0} & \cellcolor{bestbg}\textbf{79.6} & \cellcolor{bestbg}\textbf{77.1} & \cellcolor{bestbg}\textbf{77.0} \\
& DPO & 63.8 & 73.7 & 68.5 & 68.7 & 60.3 & 67.8 & 64.4 & 64.3 & 47.1 & 45.5 & 42.7 & 46.5 & 46.0 & 45.6 & 43.5 & 46.2 \\
\rowcolor{rowbg} & Bucket-SFT & \underline{67.4} & \underline{78.6} & \underline{73.7} & \underline{73.0} & \underline{63.6} & 72.0 & 68.6 & 68.0 & \underline{70.8} & \underline{79.8} & \underline{75.6} & \underline{75.3} & \underline{69.1} & \underline{75.6} & \underline{72.8} & \underline{72.5} \\
\rowcolor{rowbg} \multirow{-5}{*}{\rotatebox[origin=c]{90}{\textbf{Gemma-2-2B}}} & HDPO & \cellcolor{bestbg}\textbf{68.1} & \cellcolor{bestbg}\textbf{79.1} & \cellcolor{bestbg}\textbf{75.0} & \cellcolor{bestbg}\textbf{73.4} & \cellcolor{bestbg}\textbf{66.0} & \cellcolor{bestbg}\textbf{79.5} & \cellcolor{bestbg}\textbf{76.0} & \cellcolor{bestbg}\textbf{70.5} & 65.9 & 72.8 & 69.2 & 67.9 & 64.5 & 74.0 & 71.0 & 68.0 \\
\midrule
& BaseLM & 36.4 & 41.8 & 36.9 & 39.2 & 41.0 & 45.2 & 42.0 & 43.5 & 59.4 & 63.9 & 59.8 & 61.5 & 59.2 & 62.3 & 60.2 & 61.0 \\
& Full-SFT & \cellcolor{bestbg}\textbf{75.4} & \cellcolor{bestbg}\textbf{84.9} & \cellcolor{bestbg}\textbf{81.0} & \cellcolor{bestbg}\textbf{80.4} & \cellcolor{bestbg}\textbf{71.3} & \cellcolor{bestbg}\textbf{79.2} & \cellcolor{bestbg}\textbf{76.2} & \cellcolor{bestbg}\textbf{75.5} & \underline{84.1} & \underline{88.5} & \underline{85.9} & \underline{86.5} & \cellcolor{bestbg}\textbf{82.8} & \cellcolor{bestbg}\textbf{86.8} & \cellcolor{bestbg}\textbf{85.0} & \cellcolor{bestbg}\textbf{85.0} \\
& DPO & 64.4 & 72.5 & 68.0 & 68.6 & 61.0 & 67.0 & 64.0 & 64.3 & 72.9 & 75.7 & 72.9 & 74.1 & 70.8 & 73.5 & 72.0 & 72.4 \\
\rowcolor{rowbg} & Bucket-SFT & \underline{72.8} & \underline{81.7} & \underline{77.4} & \underline{77.3} & \underline{69.0} & \underline{76.0} & \underline{73.2} & \underline{72.8} & \cellcolor{bestbg}\textbf{84.4} & \cellcolor{bestbg}\textbf{89.0} & \cellcolor{bestbg}\textbf{86.4} & \cellcolor{bestbg}\textbf{86.8} & \underline{82.6} & \underline{86.6} & \underline{84.9} & \underline{84.8} \\
\rowcolor{rowbg} \multirow{-5}{*}{\rotatebox[origin=c]{90}{\textbf{Qwen2.5-3B}}} & HDPO & 46.9 & 55.7 & 49.8 & 50.7 & 50.5 & 60.8 & 54.5 & 55.0 & 49.9 & 56.1 & 50.2 & 51.6 & 54.0 & 64.5 & 59.0 & 58.0 \\
\midrule
& BaseLM & 52.1 & 61.9 & 56.4 & 56.9 & 51.9 & 59.2 & 55.7 & 56.0 & 68.4 & 74.0 & 70.2 & 70.9 & 66.3 & 70.3 & 68.0 & 68.5 \\
& Full-SFT & \cellcolor{bestbg}\textbf{74.3} & \underline{85.5} & \underline{81.0} & \underline{80.2} & \cellcolor{bestbg}\textbf{70.2} & \underline{79.4} & \underline{76.0} & \underline{75.1} & 75.6 & 82.8 & 79.0 & 79.3 & 74.5 & 80.3 & 77.7 & 77.6 \\
& DPO & 69.4 & 78.7 & 74.4 & 74.2 & 65.7 & 72.9 & 69.9 & 69.5 & 48.7 & 52.6 & 48.2 & 50.6 & 47.1 & 49.3 & 46.9 & 48.5 \\
\rowcolor{rowbg} & Bucket-SFT & \underline{73.0} & \cellcolor{bestbg}\textbf{86.1} & 79.4 & 78.4 & 68.4 & 76.9 & 73.8 & 72.9 & \cellcolor{bestbg}\textbf{78.6} & \underline{84.8} & \underline{81.3} & \underline{81.6} & \cellcolor{bestbg}\textbf{77.1} & \underline{82.0} & \underline{79.8} & \underline{79.7} \\
\rowcolor{rowbg} \multirow{-5}{*}{\rotatebox[origin=c]{90}{\textbf{Llama-3.2-3B}}} & HDPO & 70.8 & 83.9 & \cellcolor{bestbg}\textbf{82.0} & \cellcolor{bestbg}\textbf{80.7} & \underline{69.5} & \cellcolor{bestbg}\textbf{81.5} & \cellcolor{bestbg}\textbf{78.0} & \cellcolor{bestbg}\textbf{75.5} & 73.1 & \cellcolor{bestbg}\textbf{85.3} & \cellcolor{bestbg}\textbf{82.1} & \cellcolor{bestbg}\textbf{81.9} & 75.0 & \cellcolor{bestbg}\textbf{84.0} & \cellcolor{bestbg}\textbf{81.5} & \cellcolor{bestbg}\textbf{80.0} \\
\bottomrule
\end{tabular}%
\caption{LLM-as-a-judge results using Qwen3-30B-A3B-IT and Llama-3.3-70B-IT. Scores are evaluated on a 0-100 scale: Utility, Cond. (Conditional Naturalness), D-Faith. (Distribution Faithfulness), and Overall.
}
\label{tab:llm-judge-results}
\end{table*}

\definecolor{rowbg}{HTML}{DDE2EC}

\begin{table*}[t]
\centering
\scriptsize
\setlength{\tabcolsep}{2.7pt}
\renewcommand{\arraystretch}{1.10}
\begin{tabular}{ll*{10}{c}}
\toprule
\multirow{2}{*}{\textbf{Model}} & \multirow{2}{*}{\textbf{Method}} & \multicolumn{5}{c}{\textbf{English}} & \multicolumn{5}{c}{\textbf{Chinese}} \\
\cmidrule(lr){3-7} \cmidrule(lr){8-12}
& & \makecell{Agg\\($\uparrow$)} & \makecell{Binoc.\\AI\% ($\downarrow$)} & \makecell{Fast-DGPT\\AI\% ($\downarrow$)} & \makecell{RADAR\\AI\% ($\downarrow$)} & \makecell{Joint\\($\uparrow$)} & \makecell{Agg\\($\uparrow$)} & \makecell{Binoc.\\AI\% ($\downarrow$)} & \makecell{Fast-DGPT\\AI\% ($\downarrow$)} & \makecell{RADAR\\AI\% ($\downarrow$)} & \makecell{Joint\\($\uparrow$)} \\
\midrule
& BaseLM & 0.235 & 62.2 & 44.5 & 98.5 & 27.0 & 0.293 & 28.1 & \textbf{27.2} & 58.1 & 39.8 \\
& Full-SFT & 0.371 & \textbf{23.6} & \textbf{12.8} & \textbf{85.4} & 45.7 & 0.529 & \textbf{24.1} & \underline{26.4} & \underline{49.8} & 58.9 \\
& DPO & 0.527 & 53.6 & 27.1 & 96.8 & 46.0 & 0.184 & 40.8 & 38.3 & 59.9 & 27.4 \\
\rowcolor{rowbg} & Bucket-SFT & \textbf{0.643} & \underline{38.1} & \underline{15.1} & \underline{93.2} & \textbf{57.0} & \underline{0.684} & \underline{25.5} & 28.4 & 60.0 & \underline{65.1} \\
\rowcolor{rowbg} \multirow{-5}{*}{\rotatebox[origin=c]{90}{\textbf{Qwen2.5-1.5B}}} & HDPO & \underline{0.639} & 47.2 & 26.1 & 96.0 & \underline{51.8} & \textbf{0.694} & 37.8 & 36.7 & \textbf{35.3} & \textbf{66.3} \\
\midrule
& BaseLM & 0.048 & 55.3 & 35.4 & 96.2 & 8.5 & 0.113 & 35.3 & 35.2 & 60.0 & 18.8 \\
& Full-SFT & 0.138 & \underline{25.8} & \underline{14.4} & \textbf{84.9} & 22.3 & \textbf{0.489} & \underline{22.7} & \underline{27.5} & \underline{54.7} & \textbf{55.8} \\
& DPO & 0.043 & 54.3 & 31.0 & 95.7 & 7.8 & 0.058 & 37.0 & 43.9 & \underline{34.9} & 10.6 \\
\rowcolor{rowbg} & Bucket-SFT & \underline{0.595} & \textbf{23.8} & \textbf{7.4} & \underline{89.1} & \textbf{59.7} & \underline{0.242} & \textbf{19.3} & \textbf{25.0} & 57.9 & \underline{35.4} \\
\rowcolor{rowbg} \multirow{-5}{*}{\rotatebox[origin=c]{90}{\textbf{Gemma-2-2B}}} & HDPO & \textbf{0.731} & 42.7 & 19.7 & 95.8 & \underline{57.4} & 0.245 & 40.8 & 39.0 & \textbf{35.6} & 35.0 \\
\midrule
& BaseLM & 0.277 & \textbf{32.9} & 17.9 & 99.9 & 35.6 & 0.374 & 27.5 & \textbf{24.0} & 64.7 & 46.4 \\
& Full-SFT & 0.137 & 41.8 & \textbf{16.5} & 95.4 & 21.4 & \underline{0.498} & \textbf{23.5} & \underline{25.2} & \underline{56.6} & \underline{56.4} \\
& DPO & 0.381 & 60.0 & 43.1 & \underline{94.9} & 35.9 & 0.148 & 35.0 & 33.7 & 48.3 & 23.8 \\
\rowcolor{rowbg} & Bucket-SFT & \textbf{0.688} & \underline{39.3} & \underline{16.6} & \textbf{93.2} & \textbf{58.1} & \textbf{0.669} & \underline{24.9} & 27.4 & 59.9 & \textbf{64.7} \\
\rowcolor{rowbg} \multirow{-5}{*}{\rotatebox[origin=c]{90}{\textbf{Qwen2.5-3B}}} & HDPO & \underline{0.567} & 42.5 & 24.8 & 96.4 & \underline{50.4} & 0.223 & 34.9 & 34.0 & \textbf{32.6} & 33.4 \\
\midrule
& BaseLM & 0.278 & \textbf{45.0} & \underline{20.6} & 99.5 & 34.4 & 0.249 & \textbf{25.7} & \textbf{17.5} & 70.5 & 35.5 \\
& Full-SFT & \underline{0.550} & 49.3 & 21.5 & \underline{98.5} & \underline{48.6} & 0.433 & 32.5 & 34.2 & \underline{43.8} & 51.4 \\
& DPO & 0.288 & 51.5 & 26.9 & 99.0 & 33.8 & 0.088 & 38.3 & 46.6 & 46.9 & 15.2 \\
\rowcolor{rowbg} & Bucket-SFT & 0.566 & \underline{44.7} & \textbf{18.0} & \textbf{94.1} & \textbf{51.8} & \underline{0.636} & \underline{31.8} & \underline{31.7} & 61.1 & \underline{60.9} \\
\rowcolor{rowbg} \multirow{-5}{*}{\rotatebox[origin=c]{90}{\textbf{Llama-3.2-3B}}} & HDPO & \textbf{0.604} & 50.0 & 23.5 & 97.0 & 50.3 & \textbf{0.665} & 50.8 & 50.1 & \textbf{27.5} & \textbf{61.5} \\
\bottomrule
\end{tabular}%
\caption{AI-detection and joint quality-human-likeness results. Detector scores are AI-classification rates, where lower is better. Joint is the harmonic mean of Agg and the mean detector human-likeness score $1-\mathrm{AI\ rate}$, scaled to 0-100.}
\label{tab:ai_detection_joint_results}
\end{table*}

\begin{figure*}[t]
    \centering
    \begin{subfigure}{0.49\linewidth}
        \centering
        \includegraphics[width=\linewidth]{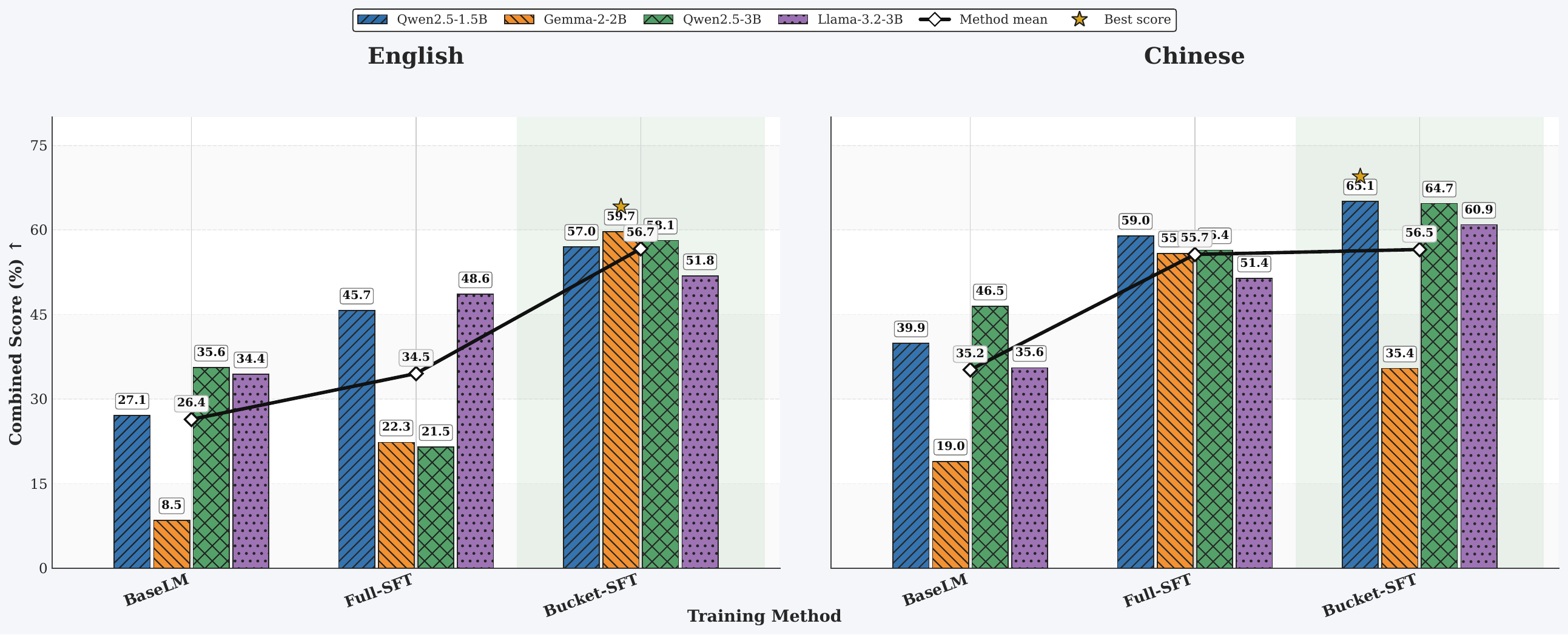}
        \vspace{-0.8em}
        \caption{BaseLM, Full-SFT, and Bucket-SFT.}
        \label{fig:quality-a}
    \end{subfigure}
    \hfill
    \begin{subfigure}{0.49\linewidth}
        \centering
        \includegraphics[width=\linewidth]{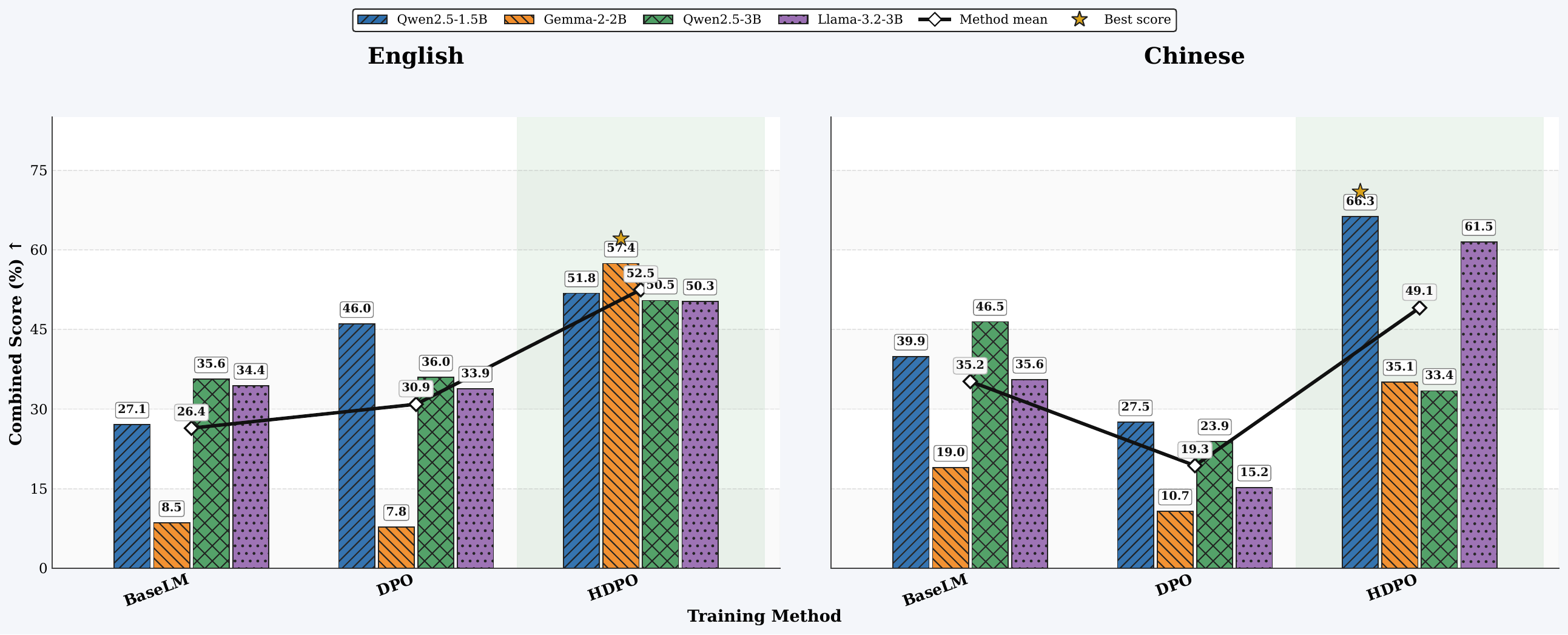}
        \vspace{-0.8em}
        \caption{BaseLM, DPO, and HDPO.}
        \label{fig:quality-b}
    \end{subfigure}
    \caption{Combined benchmark quality and detector human-likeness scores across English and Chinese models. 
    }
    \label{fig:quality-human-likeness-comparison}
\end{figure*}

\definecolor{TableHeader}{HTML}{385D8A}
\definecolor{bestdCell}{HTML}{D9EAD3}

\newcommand{\bestd}[1]{\cellcolor{bestdCell}\textbf{#1}}

\begin{table*}[t]
\centering
\scriptsize
\setlength{\tabcolsep}{3.5pt}
\begin{tabular}{llrrrrrr}
\toprule
\textbf{Model} &
\textbf{Dataset} &
\textbf{BaseLM} &
\textbf{CoT} &
\textbf{Full-SFT} &
\textbf{DPO} &
\textbf{Bucket-SFT} &
\textbf{HDPO} \\
\midrule

\multirow{8}{*}{\rotatebox{90}{\textbf{Qwen2.5-1.5B}}}
& coqa              & 0.335 & 0.044 & 0.587 & \bestd{0.617} & \second{0.613} & 0.266 \\
& dailydialog       & 0.102 & 0.052 & 0.131 & 0.133 & \second{0.150} & \bestd{0.256} \\
& dolly             & 0.319 & 0.203 & 0.289 & 0.334 & \second{0.388} & \bestd{0.609} \\
& eli5\_category    & 0.161 & 0.175 & \second{0.199} & 0.164 & \bestd{0.226} & 0.185 \\
& ms\_marco         & 0.352 & 0.132 & 0.371 & 0.409 & \second{0.595} & \bestd{0.776} \\
& multiwoz          & 0.158 & 0.081 & 0.281 & 0.279 & \bestd{0.297} & \second{0.285} \\
& natural\_questions& 0.188 & 0.238 & 0.214 & 0.288 & \second{0.294} & \bestd{0.690} \\
& squad\_v2         & 0.272 & 0.038 & 0.275 & 0.369 & \second{0.562} & \bestd{0.836} \\
\midrule

\multirow{8}{*}{\rotatebox{90}{\textbf{Gemma-2-2B}}}
& coqa              & 0.395 & 0.020 & \bestd{0.580} & \second{0.566} & 0.548 & 0.271 \\
& dailydialog       & 0.131 & 0.057 & \second{0.155} & 0.135 & 0.138 & \bestd{0.303} \\
& dolly             & 0.217 & 0.090 & 0.282 & \second{0.319} & 0.289 & \bestd{0.807} \\
& eli5\_category    & 0.191 & 0.113 & \second{0.206} & 0.177 & \bestd{0.217} & 0.184 \\
& ms\_marco         & 0.151 & 0.052 & 0.273 & 0.306 & \second{0.448} & \bestd{0.858} \\
& multiwoz          & 0.131 & 0.069 & \second{0.302} & 0.285 & 0.283 & \bestd{0.436} \\
& natural\_questions& 0.168 & 0.161 & 0.238 & 0.138 & \second{0.274} & \bestd{0.762} \\
& squad\_v2         & 0.164 & 0.008 & 0.283 & 0.344 & \second{0.505} & \bestd{0.924} \\
\midrule

\multirow{8}{*}{\rotatebox{90}{\textbf{Qwen2.5-3B}}}
& coqa              & 0.118 & 0.040 & \second{0.667} & 0.627 & \bestd{0.676} & 0.278 \\
& dailydialog       & 0.011 & 0.054 & \second{0.149} & 0.136 & \second{0.149} & \bestd{0.447} \\
& dolly             & 0.218 & 0.216 & 0.336 & 0.213 & \second{0.390} & \bestd{0.566} \\
& eli5\_category    & 0.150 & 0.182 & \second{0.191} & 0.105 & \bestd{0.229} & 0.165 \\
& ms\_marco         & 0.155 & 0.139 & 0.505 & 0.158 & \bestd{0.584} & \second{0.569} \\
& multiwoz          & 0.019 & 0.080 & \second{0.306} & 0.264 & 0.297 & \bestd{0.325} \\
& natural\_questions& 0.176 & 0.259 & 0.266 & 0.217 & \second{0.303} & \bestd{0.748} \\
& squad\_v2         & \second{0.495} & 0.031 & 0.425 & 0.255 & \bestd{0.661} & 0.399 \\
\midrule

\multirow{8}{*}{\rotatebox{90}{\textbf{Llama-3.2-3B}}}
& coqa              & 0.438 & 0.029 & \second{0.650} & \bestd{0.674} & 0.474 & 0.511 \\
& dailydialog       & 0.118 & 0.052 & \second{0.151} & 0.140 & 0.149 & \bestd{0.169} \\
& dolly             & 0.184 & 0.181 & 0.340 & 0.301 & \second{0.365} & \bestd{0.400} \\
& eli5\_category    & 0.139 & 0.174 & 0.185 & 0.180 & \second{0.219} & \bestd{0.245} \\
& ms\_marco         & 0.161 & 0.114 & 0.411 & 0.425 & \second{0.533} & \bestd{0.569} \\
& multiwoz          & 0.134 & 0.071 & \second{0.313} & 0.283 & 0.306 & \bestd{0.338} \\
& natural\_questions& 0.062 & 0.223 & 0.177 & 0.165 & \second{0.307} & \bestd{0.340} \\
& squad\_v2         & 0.344 & 0.080 & 0.310 & 0.366 & \second{0.384} & \bestd{0.419} \\
\bottomrule
\end{tabular}
\caption{Per-dataset QA-F1 scores on English benchmarks. Green and blue indicate the best and second-best results within each dataset, respectively.}
\label{tab:dataset_en}
\end{table*}
\begin{table*}[t]
\centering
\scriptsize
\setlength{\tabcolsep}{3.5pt}
\begin{tabular}{llrrrrrr}
\toprule
\textbf{Model} &
\textbf{Dataset} &
\textbf{BaseLM} &
\textbf{CoT} &
\textbf{Full-SFT} &
\textbf{DPO} &
\textbf{Bucket-SFT} &
\textbf{HDPO} \\
\midrule

\multirow{5}{*}{\rotatebox{90}{\textbf{Qwen2.5-1.5B}}}
& cmrc2018     & 0.134 & 0.028 & 0.221 & 0.146 & \second{0.364} & \bestd{0.519} \\
& coig\_cqia   & 0.069 & 0.032 & \second{0.101} & 0.058 & \second{0.101} & \bestd{0.228} \\
& drcd         & 0.304 & 0.023 & 0.566 & 0.382 & \second{0.745} & \bestd{0.912} \\
& hc3\_chinese & 0.006 & 0.004 & \second{0.010} & 0.001 & \second{0.010} & \bestd{0.201} \\
& oasst2\_zh   & 0.011 & 0.037 & 0.035 & 0.014 & \second{0.045} & \bestd{0.160} \\
\midrule

\multirow{5}{*}{\rotatebox{90}{\textbf{Gemma-2-2B}}}
& cmrc2018     & 0.065 & 0.008 & 0.172 & 0.047 & \second{0.216} & \bestd{0.548} \\
& coig\_cqia   & 0.046 & 0.003 & \second{0.095} & 0.031 & 0.076 & \bestd{0.170} \\
& drcd         & 0.170 & 0.014 & \second{0.572} & 0.128 & 0.511 & \bestd{0.905} \\
& hc3\_chinese & 0.004 & 0.002 & \second{0.011} & 0.002 & 0.010 & \bestd{0.283} \\
& oasst2\_zh   & 0.007 & 0.014 & 0.028 & 0.022 & \second{0.029} & \bestd{0.196} \\
\midrule

\multirow{5}{*}{\rotatebox{90}{\textbf{Qwen2.5-3B}}}
& cmrc2018     & 0.160 & 0.061 & 0.325 & 0.084 & \second{0.429} & \bestd{0.488} \\
& coig\_cqia   & 0.069 & 0.033 & 0.094 & 0.064 & \second{0.125} & \bestd{0.220} \\
& drcd         & 0.308 & 0.072 & 0.699 & 0.552 & \second{0.765} & \bestd{0.897} \\
& hc3\_chinese & 0.005 & 0.006 & 0.011 & 0.002 & \second{0.012} & \bestd{0.195} \\
& oasst2\_zh   & 0.010 & 0.030 & 0.031 & 0.026 & \second{0.037} & \bestd{0.495} \\
\midrule

\multirow{5}{*}{\rotatebox{90}{\textbf{Llama-3.2-3B}}}
& cmrc2018     & 0.176 & 0.013 & 0.122 & 0.067 & \second{0.329} & \bestd{0.630} \\
& coig\_cqia   & 0.040 & 0.015 & \second{0.091} & 0.067 & 0.085 & \bestd{0.138} \\
& drcd         & 0.483 & 0.014 & 0.428 & 0.155 & \second{0.725} & \bestd{0.888} \\
& hc3\_chinese & 0.004 & 0.003 & \second{0.011} & 0.003 & 0.010 & \bestd{0.169} \\
& oasst2\_zh   & 0.011 & 0.015 & \second{0.037} & 0.023 & 0.025 & \bestd{0.156} \\
\bottomrule
\end{tabular}
\caption{Per-dataset QA-F1 scores on Chinese benchmarks. Green and blue indicate the best and second-best results within each dataset, respectively.}
\label{tab:dataset_zh}
\end{table*}

\begin{table*}[t]
\centering
\scriptsize
\setlength{\tabcolsep}{5pt}
\renewcommand{\arraystretch}{1.08}
\begin{tabular}{@{}lllccccc@{}}
\toprule
\textbf{Model} &
\textbf{Lang.} &
\textbf{Configuration} &
\textbf{QA-F1} $\uparrow$ &
\textbf{BNG} $\downarrow$ &
\textbf{G-MAUVE} $\uparrow$ &
\textbf{NUF} $\uparrow$ &
\textbf{Agg} $\uparrow$ \\
\midrule

\multirow{4}{*}{\textbf{Qwen2.5-1.5B}}
& \multirow{2}{*}{EN}
& Bucket-SFT + DPO
& \best{0.449}
& \best{0.410}
& \best{0.936}
& \best{0.594}
& \best{0.649} \\

&
& Full-SFT + HDPO
& \second{0.323}
& \second{5.142}
& \second{0.919}
& \second{0.558}
& \second{0.405} \\

\cmidrule(lr){2-8}

& \multirow{2}{*}{ZH}
& Bucket-SFT + DPO
& \best{0.475}
& \best{0.377}
& \best{0.956}
& \best{0.666}
& \best{0.685} \\

&
& Full-SFT + HDPO
& \second{0.349}
& \second{0.939}
& \second{0.883}
& \second{0.517}
& \second{0.535} \\

\midrule

\multirow{4}{*}{\textbf{Gemma-2-2B}}
& \multirow{2}{*}{EN}
& Bucket-SFT + DPO
& \best{0.418}
& \best{0.407}
& \second{0.870}
& \second{0.578}
& \best{0.622} \\

&
& Full-SFT + HDPO
& \second{0.354}
& \second{372.816}
& \best{0.886}
& \best{0.635}
& \second{0.152} \\

\cmidrule(lr){2-8}

& \multirow{2}{*}{ZH}
& Bucket-SFT + DPO
& \second{0.362}
& \second{39.129}
& \best{0.949}
& \best{0.431}
& \second{0.246} \\

&
& Full-SFT + HDPO
& \best{0.421}
& \best{1.233}
& \second{0.911}
& \second{0.341}
& \best{0.492} \\

\midrule

\multirow{4}{*}{\textbf{Qwen2.5-3B}}
& \multirow{2}{*}{EN}
& Bucket-SFT + DPO
& \best{0.472}
& \best{0.215}
& \best{0.898}
& \best{0.648}
& \best{0.690} \\

&
& Full-SFT + HDPO
& \second{0.381}
& \second{794.784}
& \second{0.881}
& \second{0.591}
& \second{0.126} \\

\cmidrule(lr){2-8}

& \multirow{2}{*}{ZH}
& Bucket-SFT + DPO
& \best{0.498}
& \best{0.630}
& \best{0.958}
& \best{0.641}
& \best{0.658} \\

&
& Full-SFT + HDPO
& \second{0.431}
& \second{8.887}
& \second{0.884}
& \second{0.432}
& \second{0.359} \\

\midrule

\multirow{4}{*}{\textbf{Llama-3.2-3B}}
& \multirow{2}{*}{EN}
& Bucket-SFT + DPO
& \best{0.379}
& \best{0.406}
& \second{0.863}
& \second{0.471}
& \best{0.575} \\

&
& Full-SFT + HDPO
& \second{0.354}
& \second{1.127}
& \best{0.871}
& \best{0.675}
& \second{0.559} \\

\cmidrule(lr){2-8}

& \multirow{2}{*}{ZH}
& Bucket-SFT + DPO
& \best{0.455}
& \best{0.651}
& \best{0.938}
& \best{0.658}
& \best{0.642} \\

&
& Full-SFT + HDPO
& \second{0.262}
& \second{0.829}
& \second{0.807}
& \second{0.371}
& \second{0.455} \\

\bottomrule
\end{tabular}%
\caption{Ablations comparing Bucket-SFT + DPO with Full-SFT + HDPO. 
Green and blue indicate the best and second-best results within each model-language pair, respectively.}
\label{tab:ablations}
\end{table*}
\definecolor{panelbg}{HTML}{DDE2EC}

\begin{table*}[t]
\centering
\scriptsize
\setlength{\tabcolsep}{4.5pt}
\renewcommand{\arraystretch}{1.08}
\resizebox{\textwidth}{!}{%
\begin{tabular}{@{}llrrrrlrr@{}}
\toprule
\textbf{Model} &
\textbf{Method} &
\textbf{Macro} &
\textbf{Weighted} &
\textbf{Median} &
\textbf{Maximum} &
\textbf{Maximizing bucket} &
\(\boldsymbol{n}\) &
\textbf{Macro w-o max} \\
\midrule

\rowcolor{panelbg}
\multicolumn{9}{c}{\textbf{English (EN)}} \\
\midrule

\multirow{3}{*}{\textbf{Gemma-2-2B}}
& BaseLM
& 7708.600 & 4256.915 & 1.330 & 40694.950
& \textit{single-qa-long} & 3314 & 4409.996 \\

& Full-SFT
& 411.770 & 364.030 & 0.927 & 1907.894
& \textit{single-qa-long} & 3314 & 262.164 \\

& DPO
& 28526.000 & 3281.397 & 3.527 & 280630.663
& \textit{single-assistant-long} & 79 & 3316.354 \\

\addlinespace[2pt]
\multirow{1}{*}{\textbf{Qwen2.5-3B}}
& Full-SFT
& 649.490 & 73.861 & 4.634 & 6460.159
& \textit{single-assistant-medium} & 246 & 68.432 \\

\addlinespace[2pt]
\multirow{1}{*}{\textbf{Llama-3.2-3B}}
& DPO
& 24.431 & 5.850 & 0.839 & 101.133
& \textit{single-assistant-medium} & 246 & 16.761 \\

\midrule
\rowcolor{panelbg}
\multicolumn{9}{c}{\textbf{Chinese (ZH)}} \\
\midrule

\multirow{1}{*}{\textbf{Qwen2.5-1.5B}}
& DPO
& 16.056 & 5.298 & 11.483 & 64.244
& \textit{multi-dialogue-short} & 66 & 11.237 \\

\addlinespace[2pt]
\multirow{4}{*}{\textbf{Gemma-2-2B}}
& BaseLM
& 13.152 & 2.581 & 1.860 & 108.624
& \textit{single-assistant-medium} & 48 & 3.605 \\

& DPO
& 36.750 & 82.293 & 32.517 & 105.137
& \textit{single-qa-short} & 4405 & 29.912 \\

& Bucket-SFT
& 39.365 & 2.760 & 0.362 & 428.342
& \textit{multi-dialogue-long} & 37 & 0.468 \\

& HDPO
& 23.494 & 23.333 & 23.514 & 23.708
& \textit{multi-dialogue-short} & 66 & 23.473 \\

\addlinespace[2pt]
\multirow{2}{*}{\textbf{Qwen2.5-3B}}
& DPO
& 37.530 & 27.211 & 31.530 & 74.351
& \textit{multi-dialogue-long} & 37 & 33.848 \\

& HDPO
& 23.537 & 23.476 & 23.580 & 23.914
& \textit{single-qa-long} & 516 & 23.500 \\

\addlinespace[2pt]
\multirow{2}{*}{\textbf{Llama-3.2-3B}}
& BaseLM
& 18.230 & 10.762 & 6.485 & 122.713
& \textit{single-assistant-xlong} & 456 & 7.789 \\

& DPO
& 31.843 & 32.419 & 31.918 & 48.785
& \textit{multi-dialogue-short} & 66 & 30.149 \\

\bottomrule
\end{tabular}%
}
\caption{Bucket diagnostics for main-table BNG values above 10. ``Macro w-o max'' removes the highest-BNG bucket, revealing whether the aggregate deviation is broad or concentrated. 
}
\label{tab:bng_diagnostics}
\end{table*}

\begin{table*}[t]
\centering\scriptsize
\setlength{\tabcolsep}{4pt}
\renewcommand{\arraystretch}{1.10}
\begin{tabularx}{\textwidth}{@{}l r X@{}}
\toprule
Feature family & Count & Exact feature names \\
\midrule
LM likelihood and truncation & 43 & \textit{lm-answer-char-count, lm-conditional-answer-token-count, lm-conditional-available, lm-conditional-bits-per-token, lm-conditional-condition-text-available, lm-conditional-lm-answer-tokens-dropped, lm-conditional-lm-prefix-tokens-dropped, lm-conditional-lm-truncated, lm-conditional-logprob-gain, lm-conditional-logprob-mean, lm-conditional-logprob-sum, lm-conditional-nll-mean, lm-conditional-nll-sum, lm-conditional-perplexity, lm-conditional-perplexity-ratio, lm-conditional-prefix-token-count, lm-conditional-score-coverage, lm-conditional-scored-token-count, lm-conditional-surprisal-max, lm-conditional-surprisal-mean, lm-conditional-surprisal-min, lm-conditional-surprisal-std, lm-max-seq-length, lm-prefix-char-count, lm-unconditional-answer-token-count, lm-unconditional-available, lm-unconditional-bits-per-token, lm-unconditional-condition-text-available, lm-unconditional-lm-answer-tokens-dropped, lm-unconditional-lm-prefix-tokens-dropped, lm-unconditional-lm-truncated, lm-unconditional-logprob-mean, lm-unconditional-logprob-sum, lm-unconditional-nll-mean, lm-unconditional-nll-sum, lm-unconditional-perplexity, lm-unconditional-prefix-token-count, lm-unconditional-score-coverage, lm-unconditional-scored-token-count, lm-unconditional-surprisal-max, lm-unconditional-surprisal-mean, lm-unconditional-surprisal-min, lm-unconditional-surprisal-std} \\
Morphosyntax & 39 & \textit{base-verb-count, content-word-count, content-word-ratio, contraction-count, gerund-verb-count, modal-count, modal-ratio, non-3sg-present-verb-count, past-participle-count, past-verb-count, pos-adj-count, pos-adj-ratio, pos-adp-count, pos-adp-ratio, pos-adv-count, pos-adv-ratio, pos-conj-count, pos-conj-ratio, pos-det-count, pos-det-ratio, pos-noun-count, pos-noun-ratio, pos-num-count, pos-num-ratio, pos-pron-count, pos-pron-ratio, pos-prt-count, pos-prt-ratio, pos-verb-count, pos-verb-ratio, pos-x-count, pos-x-ratio, proper-noun-count, proper-noun-ratio, stopword-count, stopword-ratio, third-person-present-verb-count, wh-word-count, wh-word-ratio} \\
Readability and word complexity & 28 & \textit{automated-readability-index, avg-char-per-sentence, avg-char-per-token, avg-sentence-char-length, avg-syllables-per-word, avg-tokens-per-sentence, avg-word-length, coleman-liau, flesch-kincaid-grade, flesch-reading-ease, gunning-fog, long-word-count, long-word-ratio, max-sentence-char-length, max-tokens-per-sentence, medium-word-count, min-sentence-char-length, min-tokens-per-sentence, multisyllabic-word-count, multisyllabic-word-ratio, short-word-count, short-word-ratio, smog-index, std-sentence-char-length, std-tokens-per-sentence, syllable-count, very-long-word-count, very-long-word-ratio} \\
Lexical diversity and repetition & 17 & \textit{adjacent-repeat-count, corrected-ttr, dislegomena-ratio, distinct-1, distinct-2, distinct-3, hapax-ratio, lexical-density, max-token-frequency, max-token-frequency-ratio, repeated-bigram-ratio, repeated-char-sequence-count, repeated-token-ratio, root-ttr, shannon-entropy, type-token-ratio, unique-word-count} \\
Formatting and punctuation & 26 & \textit{apostrophe-count, bracket-count, bullet-line-count, code-fence-count, colon-count, comma-count, ellipsis-count, email-count, enumerated-line-count, exclam-sentence-ratio, exclamation-mark-count, hyphen-count, inline-code-marker-count, newline-count, paragraph-count, parenthesis-count, period-count, punctuation-count, punctuation-density, question-mark-count, question-sentence-ratio, quote-count, semicolon-count, slash-count, terminal-punctuation-ratio, url-count} \\
Affect & 4 & \textit{vader-compound, vader-neg, vader-neu, vader-pos} \\
Surface form, length, and casing & 20 & \textit{all-caps-token-count, all-caps-token-ratio, alphabetic-char-count, char-count, char-no-space-count, digit-char-count, lowercase-char-count, lowercase-token-count, non-ascii-char-count, number-token-count, ordinal-token-count, sentence-count, sentence-initial-capital-ratio, title-case-token-count, title-case-token-ratio, token-count, uppercase-char-count, whitespace-count, word-token-count, year-token-count} \\
\bottomrule
\end{tabularx}
\caption{Complete 177-dimensional linguistic feature inventory used by the bucket support and BNG pipelines, grouped only for readability. The mix of bounded ratios, sparse indicators, counts, and likelihood-derived quantities gives broad coverage, but also explains why variance normalization must be robust: near-constant human features can otherwise dominate BNG despite being linguistically peripheral.}
\label{tab:linguistic_features}
\end{table*}
\clearpage
\begin{minipage}{\columnwidth}

\section{LLM-as-a-Judge Prompt and Rubric}
\label{app:llm-judge-prompt-rubric}

\begin{tcolorbox}[
  width=\linewidth,
  colback=polylight,
  colframe=polyblue,
  coltitle=white,
  colbacktitle=polyblue,
  title=\textbf{LLM-as-a-Judge Prompt},
  fonttitle=\bfseries\footnotesize,
  fontupper=\footnotesize,
  boxrule=0.8pt,
  arc=2mm,
  outer arc=2mm,
  left=5pt,
  right=5pt,
  top=5pt,
  bottom=5pt,
  before skip=3pt,
  after skip=0pt
]

\setlength{\parindent}{0pt}
\setlength{\parskip}{0pt}
\setlength{\emergencystretch}{1em}

\textit{You are a careful bilingual NLP evaluator. Judge the candidate response
against the user request, context, dialogue history, human reference, and
PolyAlign bucket metadata.}

\smallskip
\textbf{\textcolor{polyteal}{Evaluation constraints.}}
Do not reveal chain-of-thought. Return one valid JSON object only. Use the full
input text provided by the caller; do not ignore or summarize long context.

\smallskip
\textbf{\textcolor{polyteal}{PolyAlign evaluation objective.}}
The candidate should preserve task utility while matching the human response
distribution appropriate to the current metadata bucket. In other words, the
response should be the right kind of answer for the right language, interaction
track, response family, style bucket, and length bin.

\smallskip
\textbf{\textcolor{polyteal}{Scoring scale.}}
Score every dimension with an integer from 1 to 5:
\textit{1 = severe failure, 2 = weak, 3 = acceptable, 4 = strong, and
5 = excellent.}

\smallskip
\textbf{\textcolor{polyteal}{Required JSON output.}}
Return JSON with exactly the following top-level keys:

\begin{tcolorbox}[
  width=\linewidth,
  colback=white,
  colframe=polygray!45,
  boxrule=0.4pt,
  arc=1.5mm,
  outer arc=1.5mm,
  left=3pt,
  right=3pt,
  top=3pt,
  bottom=3pt,
  before skip=3pt,
  after skip=3pt
]
\begin{Verbatim}[fontsize=\tiny]
{
  "scores": {
    "task_success": int,
    "factual_grounding": int,
    "instruction_following": int,
    "reference_alignment": int,
    "conditional_appropriateness": int,
    "response_shape_and_length": int,
    "discourse_naturalness": int,
    "safety": int
  },
  "major_errors": [string],
  "rationale": string
}
\end{Verbatim}
\end{tcolorbox}

\smallskip
\textbf{\textcolor{polyteal}{Input fields.}}

\begin{tcolorbox}[
  width=\linewidth,
  colback=white,
  colframe=polyblue!35,
  boxrule=0.4pt,
  arc=1.5mm,
  outer arc=1.5mm,
  left=3pt,
  right=3pt,
  top=3pt,
  bottom=3pt,
  before skip=3pt,
  after skip=3pt
]
\scriptsize

\textbf{Metadata:}

\texttt{\{language, dataset, track, family,}\\
\texttt{style\_bucket, length\_bin, bucket\_id, source\_id}\\

\smallskip
\textbf{Conversation history:}

\texttt{<history> ... </history>}

\smallskip
\textbf{Context:}

\texttt{<context> ... </context>}

\smallskip
\textbf{User request:}

\texttt{<instruction> ... </instruction>}

\smallskip
\textbf{Human reference response:}

\texttt{<reference> ... </reference>}

\smallskip
\textbf{Candidate response to judge:}

\texttt{<candidate> ... </candidate>}

\end{tcolorbox}

\smallskip
\textbf{\textcolor{polyteal}{Final instruction.}}
\textit{Now judge only the candidate response. Return valid JSON only.}

\end{tcolorbox}

\end{minipage}
\clearpage

\section{Human-Rater Prompts and Rubrics}
\label{app:human-rater-prompts}
\begin{small}
    
\begin{center}

\begin{tcolorbox}[
    enhanced jigsaw,
    breakable,
  colback=polylight,
  colframe=polyblue,
  coltitle=white,
  colbacktitle=polyblue,
  title=\textbf{Human-Judge Prompt},
  fonttitle=\bfseries,
  boxrule=0.9pt,
  arc=3mm,
  outer arc=3mm,
  left=6pt,
  right=6pt,
  top=6pt,
  bottom=6pt,
  before skip=6pt,
  after skip=8pt,
  boxed title style={
    arc=2mm,
    outer arc=2mm,
    colframe=polyblue,
    colback=polyblue
  }
]

\textit{You will evaluate one candidate response to a user request. Read the
user request, any supplied context and dialogue history, and the human
reference. Evaluate only the highlighted candidate response. The model and
training method that produced it are hidden.}

\medskip
\noindent\textbf{\textcolor{polyteal}{Evaluation principles.}}
Do not reward a response merely for matching the reference word for word;
correct paraphrases and alternative valid answers are acceptable. Judge
correctness and usefulness separately from whether the response has an
appropriate style, discourse form, and length.

\medskip
\noindent\textbf{\textcolor{polyteal}{Scoring scale.}}
For each criterion, assign an integer score from 1 to 5:

\begin{itemize}
  \item \textbf{1:} Severe failure.
  \item \textbf{2:} Major weakness.
  \item \textbf{3:} Acceptable, but with noticeable problems.
  \item \textbf{4:} Strong, with only minor problems.
  \item \textbf{5:} Excellent, with no material problem.
\end{itemize}

\medskip
\noindent\textbf{\textcolor{polyteal}{Evaluation criteria.}}

\begin{itemize}
  \item \textbf{Task success:}
  Does the response answer the request and provide the information or action
  the user needs?

  \item \textbf{Factual grounding:}
  Are its factual claims supported by the supplied context or reference, or
  otherwise accurate? Do not penalize the absence of facts that are unnecessary
  for the task.

  \item \textbf{Instruction following:}
  Does the response respect the explicit constraints, requested language,
  requested format, and dialogue context?

  \item \textbf{Reference alignment:}
  Is its substantive content compatible with the human reference, while
  allowing correct paraphrases and alternative valid formulations?

  \item \textbf{Response quality:}
  Is the response clear, coherent, natural, and appropriately detailed for the
  interaction, without being materially incomplete or unnecessarily verbose?
\end{itemize}

\medskip
\noindent\textbf{\textcolor{polyteal}{Safety and validity.}}
Flag a safety concern only when the candidate response itself introduces
harmful, unsafe, or disallowed content. When assigning a score of 1 or 2, use
the optional comment field to identify the specific error. If the item cannot
be evaluated because essential context is missing, select
\emph{invalid item} rather than inferring the missing information.

\medskip
\noindent\textbf{\textcolor{polyteal}{Final instruction.}}
\textit{Judge only the highlighted candidate response and submit one score for
each criterion.}

\end{tcolorbox}

\end{center}
\end{small}
\newpage
\begin{center}
\begin{tcolorbox}[
  enhanced jigsaw,
  breakable,
  colback=polylight,
  colframe=polyblue,
  coltitle=white,
  colbacktitle=polyblue,
  title=\textbf{Bucket-Coherence Evaluation Prompt},
  fonttitle=\bfseries,
  boxrule=0.9pt,
  arc=3mm,
  outer arc=3mm,
  left=6pt,
  right=6pt,
  top=6pt,
  bottom=6pt,
  before skip=6pt,
  after skip=6pt,
  boxed title style={
    arc=2mm,
    outer arc=2mm,
    colframe=polyblue,
    colback=polyblue
  }
]

\textit{This task evaluates interaction form rather than factual correctness.
Read the conversational situation and the candidate response, then judge
whether the response is appropriate for that setting.}

\medskip
\noindent\textbf{\textcolor{polyteal}{Required judgments.}}

Rate the following two dimensions from 1 to 5:

\begin{itemize}
  \item \textbf{Situational fit:}
  How well does the response fit the conversational situation and expected
  interaction pattern?

  \item \textbf{Form coherence:}
  Do its style, discourse structure, response family, and length form a
  coherent response for that situation?
\end{itemize}

\medskip
\noindent\textbf{\textcolor{polyteal}{Bucket selection.}}
After assigning the two ratings, select the best-matching bucket description
from the provided alternatives. All alternatives share the same language and
interaction track. Do not base the decision on topic words alone. Instead,
consider the response family, amount of detail, discourse form, style, length,
and expected interaction pattern.

\medskip
\noindent\textbf{\textcolor{polyteal}{Final instruction.}}
\textit{Evaluate the response's interaction form independently of its factual
correctness, then select the single most appropriate bucket.}

\end{tcolorbox}

\end{center}
\clearpage
\onecolumn
\section{Proofs of Theorems}
\subsection{Proof of Theorem~\ref{thm:dist_sft_macro}: Bucket-SFT Optimizes Exact Macro Bucket Risk}
\label{app:proof_dist_sft_macro}

\begin{proof}
By definition, every example in bucket $b$ receives the same pre-serialization weight
\[
\widetilde w_b = \frac{N}{|\Bset|\, n_b}.
\]
Therefore, for any fixed bucket $b \in \Bset$,
\[
\sum_{i:\, b_i=b} \widetilde w_{b_i}
=
\sum_{i:\, b_i=b} \frac{N}{|\Bset|\, n_b}
=
n_b \cdot \frac{N}{|\Bset|\, n_b}
=
\frac{N}{|\Bset|}.
\]
This proves the equal-mass identity.

Now sum over all buckets:
\[
\sum_{i=1}^N \widetilde w_{b_i}
=
\sum_{b \in \Bset}\sum_{i:\, b_i=b}\widetilde w_{b_i}
=
\sum_{b \in \Bset} \frac{N}{|\Bset|}
=
N.
\]
Hence
\[
\Ldist(\theta)
=
\frac{1}{N}
\sum_{i=1}^N \widetilde w_{b_i}\,\ell_i(\theta).
\]
Grouping by bucket gives
\[
\Ldist(\theta)
=
\frac{1}{N}
\sum_{b \in \Bset}
\sum_{i:\, b_i=b}
\frac{N}{|\Bset|\, n_b}\,\ell_i(\theta)
=
\frac{1}{|\Bset|}
\sum_{b \in \Bset}
\frac{1}{n_b}
\sum_{i:\, b_i=b}\ell_i(\theta),
\]
which is exactly the macro average of bucket-wise empirical risks.
\end{proof}
\subsection{Proof of Theorem~\ref{thm:support_distance}: Bucket-Support Distance Is a Continuous Relaxation of Human Membership}
\label{app:proof_support_distance}

\begin{proof}
Fix bucket $b$ and a feature vector $z$ with $J_b(z)\neq \varnothing$.

For each $j \in J_b(z)$, define
\[
d_{bj}(z_j)
:=
\frac{\pos{l_{bj}-z_j}+\pos{z_j-u_{bj}}}{s_{bj}}.
\]
Since $s_{bj}>0$ by definition and each positive-part term is nonnegative,
\[
d_{bj}(z_j)\ge 0
\qquad \text{for every } j \in J_b(z).
\]
Averaging over $j$ gives
\[
D_b(z)=\frac{1}{|J_b(z)|}\sum_{j\in J_b(z)} d_{bj}(z_j)\ge 0,
\]
which proves (i).

For (ii), first suppose that $z_j \in [l_{bj},u_{bj}]$ for every $j\in J_b(z)$. Then for every such $j$,
\[
\pos{l_{bj}-z_j}=0,
\qquad
\pos{z_j-u_{bj}}=0,
\]
so $d_{bj}(z_j)=0$. Hence $D_b(z)=0$.

Conversely, suppose $D_b(z)=0$. Since $D_b(z)$ is the average of nonnegative numbers, every summand must be zero:
\[
d_{bj}(z_j)=0
\qquad \text{for every } j\in J_b(z).
\]
Because $s_{bj}>0$, this implies
\[
\pos{l_{bj}-z_j}+\pos{z_j-u_{bj}}=0
\qquad \text{for every } j\in J_b(z).
\]
A sum of two nonnegative terms is zero iff both terms are zero, so
\[
\pos{l_{bj}-z_j}=0
\quad \text{and} \quad
\pos{z_j-u_{bj}}=0
\qquad \text{for every } j\in J_b(z),
\]
which is equivalent to
\[
l_{bj}\le z_j \le u_{bj}
\qquad \text{for every } j\in J_b(z).
\]
This proves (ii).

Finally, each map
\[
z_j \mapsto \pos{l_{bj}-z_j}
\quad \text{and} \quad
z_j \mapsto \pos{z_j-u_{bj}}
\]
is continuous and piecewise linear. Dividing by the positive constant $s_{bj}$ preserves these properties, and finite sums and averages preserve them as well. Hence $D_b(z)$ is continuous and piecewise linear in $z$, proving (iii).
\end{proof}

\subsection{Proof of Theorem~\ref{thm:hdpo_alignment}: Distributional Alignment of the HDPO Regularizer with Sigmoid DPO}
\label{app:proof_hdpo_alignment}

\begin{proof}
Write
\[
s^+ := s_\phi(y^+,b),
\qquad
s^- := s_\phi(y^-,b),
\qquad
p_\theta := \sigmoid(\beta \Delta_\pi).
\]
Then the HDPO regularizer is
\[
R_{\mathrm{HDPO}}(\theta)
=
p_\theta s^+ + (1-p_\theta)s^-
=
s^- + p_\theta(s^+ - s^-).
\]
Differentiate with respect to $\Delta_\pi$.
Since
\[
\frac{d}{d\Delta_\pi}\sigmoid(\beta \Delta_\pi)
=
\beta\, \sigmoid(\beta \Delta_\pi)\bigl(1-\sigmoid(\beta \Delta_\pi)\bigr)
=
\beta\, p_\theta(1-p_\theta),
\]
we obtain
\[
\frac{d R_{\mathrm{HDPO}}}{d \Delta_\pi}
=
\beta\, p_\theta(1-p_\theta)\,(s^+-s^-).
\]
This proves the derivative formula.

Now suppose $s^+<s^-$. Because $\beta>0$ and $p_\theta(1-p_\theta)>0$ for all finite $\Delta_\pi$, it follows that
\[
\frac{d R_{\mathrm{HDPO}}}{d \Delta_\pi}<0.
\]
Therefore increasing $\Delta_\pi$ decreases $R_{\mathrm{HDPO}}$, so minimizing the regularizer favors larger chosen-vs-rejected policy margin.

Next consider the sigmoid DPO loss
\[
L_{\mathrm{DPO}}(\Delta_\pi;\Delta_{\mathrm{ref}})
=
-\log \sigmoid\!\bigl(\beta(\Delta_\pi-\Delta_{\mathrm{ref}})\bigr).
\]
Let
\[
u := \beta(\Delta_\pi-\Delta_{\mathrm{ref}}).
\]
Then
\[
L_{\mathrm{DPO}} = -\log \sigmoid(u).
\]
Using
\[
\frac{d}{du}[-\log \sigmoid(u)]
=
-(1-\sigmoid(u)),
\]
and the chain rule $\frac{du}{d\Delta_\pi}=\beta$, we get
\[
\frac{d L_{\mathrm{DPO}}}{d \Delta_\pi}
=
-\beta\bigl(1-\sigmoid(u)\bigr)
=
-\beta\Bigl(1-\sigmoid\bigl(\beta(\Delta_\pi-\Delta_{\mathrm{ref}})\bigr)\Bigr)
<0.
\]
Thus increasing $\Delta_\pi$ also decreases the sigmoid DPO loss.

Hence, whenever $s^+<s^-$, both the HDPO regularizer and the sigmoid DPO term are reduced by increasing $\Delta_\pi$. They are therefore distributionally aligned during optimization.
\end{proof}
\end{document}